\documentclass{article}
\PassOptionsToPackage{numbers}{natbib}
\usepackage[final]{nips_2017}

\usepackage[utf8]{inputenc} 
\usepackage[T1]{fontenc}    
\usepackage{hyperref}       
\usepackage{url}            
\usepackage{array,colortbl,multirow,multicol,booktabs,ctable}
\usepackage{amsfonts}       
\usepackage{nicefrac}       
\usepackage{microtype}      
\usepackage{dsfont}
\usepackage{ifthen}
\usepackage{amsmath}
\usepackage{amsopn}
\usepackage{wrapfig}
\usepackage{amsthm}
\usepackage{algorithm}
\usepackage{algorithmic}

\newtheorem{proposition}{Proposition}

\newcommand{\field}[1]{\mathbb{#1}}
\newcommand{\N}{\field{N}}

\newcommand{\R}{\field{R}}

\newcommand{\vect}[1]{\boldsymbol{#1}} 
\newcommand{\mat}[1]{\boldsymbol{#1}} 
\newcommand{\tvect}[1]{\tilde{\boldsymbol{#1}}}
\newcommand{\tmat}[1]{\tilde{\boldsymbol{#1}}}
\newcommand{\tscal}[1]{\tilde{#1}}
\newcommand{\hvect}[1]{\hat{\boldsymbol{#1}}}
\newcommand{\hmat}[1]{\hat{\boldsymbol{#1}}}
\newcommand{\hscal}[1]{\hat{#1}}
\newcommand{\bvect}[1]{\bar{\boldsymbol{#1}}}
\newcommand{\bmat}[1]{\bar{\boldsymbol{#1}}}
\newcommand{\bscal}[1]{\bar{#1}}
\newcommand{\vzero}{\vect{0}}
\newcommand{\vone}{\vect{1}}
\newcommand{\mxzero}{\mat{0}}

\newcommand{\dummystring}{QWERTYU}
\newcommand{\vci}[3][\dummystr]{\ifthenelse{\equal{#1}{\dummystring}}{\vect{#2}_{#3}}{\vect{#2}_{#3}^{(#1)}}}
\newcommand{\mx}[3][\dummystr]{\ifthenelse{\equal{#1}{\dummystring}}{\mat{#2}_{#3}}{\mat{#2}_{#3}^{(#1)}}}
\newcommand{\tvci}[3][\dummystr]{\ifthenelse{\equal{#1}{\dummystring}}{\tvect{#2}_{#3}}{\tvect{#2}_{#3}^{(#1)}}}
\newcommand{\tmx}[3][\dummystr]{\ifthenelse{\equal{#1}{\dummystring}}{\tmat{#2}_{#3}}{\tmat{#2}_{#3}^{(#1)}}}
\newcommand{\tscl}[3][\dummystr]{\ifthenelse{\equal{#1}{\dummystring}}{\tscal{#2}_{#3}}{\tscal{#2}_{#3}^{(#1)}}}
\newcommand{\hvci}[3][\dummystr]{\ifthenelse{\equal{#1}{\dummystring}}{\hvect{#2}_{#3}}{\hvect{#2}_{#3}^{(#1)}}}
\newcommand{\hmx}[3][\dummystr]{\ifthenelse{\equal{#1}{\dummystring}}{\hmat{#2}_{#3}}{\hmat{#2}_{#3}^{(#1)}}}
\newcommand{\hscl}[3][\dummystr]{\ifthenelse{\equal{#1}{\dummystring}}{\hscal{#2}_{#3}}{\hscal{#2}_{#3}^{(#1)}}}
\newcommand{\bvci}[3][\dummystr]{\ifthenelse{\equal{#1}{\dummystring}}{\bvect{#2}_{#3}}{\bvect{#2}_{#3}^{(#1)}}}
\newcommand{\bmx}[3][\dummystr]{\ifthenelse{\equal{#1}{\dummystring}}{\bmat{#2}_{#3}}{\bmat{#2}_{#3}^{(#1)}}}
\newcommand{\bscl}[3][\dummystr]{\ifthenelse{\equal{#1}{\dummystring}}{\bscal{#2}_{#3}}{\bscal{#2}_{#3}^{(#1)}}}

\DeclareMathOperator{\trace}{tr}
\DeclareMathOperator{\diag}{diag}
\DeclareMathOperator*{\argmax}{arg max}
\DeclareMathOperator*{\argmin}{arg min}

\newcommand{\Ex}{\mathrm{E}}
\newcommand{\Var}{\mathrm{Var}}
\newcommand{\Cov}{\mathrm{Cov}}

\newcommand{\Ind}[1]{\mathrm{I}_{\{#1\}}}

\newcommand{\Id}{\mat{I}}

\newcommand{\eps}{\mathrm{\varepsilon}}

\newcommand{\figref}[1]{Figure~\ref{fig:#1}}
\newcommand{\tabref}[1]{Table~\ref{tab:#1}}
\newcommand{\secref}[1]{Section~\ref{sec:#1}}
\renewcommand{\eqref}[1]{Eq.~\ref{eq:#1}}
\newcommand{\eqp}[1]{(\ref{eq:#1})}

\newcommand{\tsz}[2][\dummystring]{\tscl[#1]{z}{#2}}

\newcommand{\tsphi}[2][\dummystring]{\tscl[#1]{\phi}{#2}}

\newcommand{\hsy}[2][\dummystring]{\hscl[#1]{y}{#2}}

\newcommand{\hseps}[2][\dummystring]{\hscl[#1]{\eps}{#2}}

\newcommand{\va}[2][\dummystring]{\vci[#1]{a}{#2}}
\newcommand{\vb}[2][\dummystring]{\vci[#1]{b}{#2}}
\newcommand{\vc}[2][\dummystring]{\vci[#1]{c}{#2}}
\newcommand{\vd}[2][\dummystring]{\vci[#1]{d}{#2}}
\newcommand{\ve}[2][\dummystring]{\vci[#1]{e}{#2}}

\newcommand{\vg}[2][\dummystring]{\vci[#1]{g}{#2}}
\newcommand{\vh}[2][\dummystring]{\vci[#1]{h}{#2}}

\newcommand{\vk}[2][\dummystring]{\vci[#1]{k}{#2}}
\newcommand{\vl}[2][\dummystring]{\vci[#1]{l}{#2}}
\newcommand{\vm}[2][\dummystring]{\vci[#1]{m}{#2}}

\newcommand{\vq}[2][\dummystring]{\vci[#1]{q}{#2}}
\newcommand{\vr}[2][\dummystring]{\vci[#1]{r}{#2}}
\newcommand{\vs}[2][\dummystring]{\vci[#1]{s}{#2}}
\newcommand{\vt}[2][\dummystring]{\vci[#1]{t}{#2}}

\newcommand{\vv}[2][\dummystring]{\vci[#1]{v}{#2}}
\newcommand{\vw}[2][\dummystring]{\vci[#1]{w}{#2}}
\newcommand{\vx}[2][\dummystring]{\vci[#1]{x}{#2}}
\newcommand{\vy}[2][\dummystring]{\vci[#1]{y}{#2}}
\newcommand{\vz}[2][\dummystring]{\vci[#1]{z}{#2}}

\newcommand{\vdelta}[2][\dummystring]{\vci[#1]{\delta}{#2}}
\newcommand{\veps}[2][\dummystring]{\vci[#1]{\eps}{#2}}
\newcommand{\vth}[2][\dummystring]{\vci[#1]{\theta}{#2}}
\newcommand{\vmu}[2][\dummystring]{\vci[#1]{\mu}{#2}}

\newcommand{\vpi}[2][\dummystring]{\vci[#1]{\pi}{#2}}

\newcommand{\vkappa}[2][\dummystring]{\vci[#1]{\kappa}{#2}}
\newcommand{\vxi}[2][\dummystring]{\vci[#1]{\xi}{#2}}

\newcommand{\vsigma}[2][\dummystring]{\vci[#1]{\sigma}{#2}}

\newcommand{\tva}[2][\dummystring]{\tvci[#1]{a}{#2}}

\newcommand{\tvc}[2][\dummystring]{\tvci[#1]{c}{#2}}

\newcommand{\tvg}[2][\dummystring]{\tvci[#1]{g}{#2}}
\newcommand{\tvh}[2][\dummystring]{\tvci[#1]{h}{#2}}

\newcommand{\tvz}[2][\dummystring]{\tvci[#1]{z}{#2}}

\newcommand{\hvl}[2][\dummystring]{\hvci[#1]{l}{#2}}

\newcommand{\hvs}[2][\dummystring]{\hvci[#1]{s}{#2}}

\newcommand{\hvy}[2][\dummystring]{\hvci[#1]{y}{#2}}

\newcommand{\hveps}[2][\dummystring]{\hvci[#1]{\eps}{#2}}

\newcommand{\hvmu}[2][\dummystring]{\hvci[#1]{\mu}{#2}}

\newcommand{\bvth}[2][\dummystring]{\bvci[#1]{\theta}{#2}}
\newcommand{\bvmu}[2][\dummystring]{\bvci[#1]{\mu}{#2}}

\newcommand{\mxa}[2][\dummystring]{\mx[#1]{A}{#2}}
\newcommand{\mxb}[2][\dummystring]{\mx[#1]{B}{#2}}
\newcommand{\mxc}[2][\dummystring]{\mx[#1]{C}{#2}}
\newcommand{\mxd}[2][\dummystring]{\mx[#1]{D}{#2}}

\newcommand{\mxf}[2][\dummystring]{\mx[#1]{F}{#2}}

\newcommand{\mxi}[2][\dummystring]{\mx[#1]{I}{#2}}

\newcommand{\mxm}[2][\dummystring]{\mx[#1]{M}{#2}}

\newcommand{\mxp}[2][\dummystring]{\mx[#1]{P}{#2}}
\newcommand{\mxq}[2][\dummystring]{\mx[#1]{Q}{#2}}
\newcommand{\mxr}[2][\dummystring]{\mx[#1]{R}{#2}}
\newcommand{\mxs}[2][\dummystring]{\mx[#1]{S}{#2}}
\newcommand{\mxt}[2][\dummystring]{\mx[#1]{T}{#2}}

\newcommand{\mxv}[2][\dummystring]{\mx[#1]{V}{#2}}

\newcommand{\mxx}[2][\dummystring]{\mx[#1]{X}{#2}}

\newcommand{\mxsigma}[2][\dummystring]{\mx[#1]{\Sigma}{#2}}

\newcommand{\mxxi}[2][\dummystring]{\mx[#1]{\Xi}{#2}}

\newcommand{\hmxsigma}[2][\dummystring]{\hmx[#1]{\Sigma}{#2}}

\newcommand{\bmxsigma}[2][\dummystring]{\bmx[#1]{\Sigma}{#2}}

\newcommand{\rng}[2][1]{{#1},\dots,{#2}}

\newcommand{\dataset}[1]{\texttt{#1}}
\newcommand{\articlesSubset}{\dataset{EC-sub}}
\newcommand{\articlesAll}{\dataset{EC-all}}
\newcommand{\parts}{\dataset{Parts}}

\newcommand{\model}[1]{\texttt{#1}}
\newcommand{\ets}{\model{ETS}}
\newcommand{\snyder}{\model{NegBin}}
\newcommand{\issm}{\model{LS-pure}}
\newcommand{\issmfeatures}{\model{LS-feats}}

\def\Nc{\mathcal{N}}

\begin{document}




\title{Approximate Bayesian Inference in Linear State Space Models for Intermittent Demand Forecasting at Scale}


\author{
  Matthias Seeger\\
  \texttt{matthis@amazon.de}\\  
  \And
  Syama Rangapuram\\
  \texttt{rangapur@amazon.de}\\  
  \And
  Yuyang Wang\\
  \texttt{yuyawang@amazon.de}\\  
  \And
  David Salinas\\ 
  \texttt{dsalina@amazon.de}\\  
  \And 
  Jan Gasthaus\\
  \texttt{gasthaus@amazon.de}\\  
  \And
  Tim Januschowski\\
  \texttt{tjnsch@amazon.de}
  \And 
  Valentin Flunkert\\
  \texttt{flunkert@amazon.de}\\
  \AND
  Amazon Development Center, Germany \\
}


\maketitle

\begin{abstract}
  We present a scalable and robust Bayesian inference method for linear state space models.
  The method is applied to demand forecasting in the
  context of a large e-commerce platform, paying special attention to intermittent
  and bursty target statistics.  Inference is approximated by the Newton-Raphson
  algorithm, reduced to linear-time Kalman smoothing, which allows us to operate on
  several orders of magnitude larger problems than previous related work. In a study
  on large real-world sales datasets, our method outperforms competing approaches
  on fast and medium moving items.	
\end{abstract}

%
%
%
%



\section{Introduction}
\label{sec:intro}

Demand forecasting plays a central role in supply chain management, driving automated ordering, in-stock management, and facilities planning. Classical forecasting methods, such as exponential smoothing \citep{Hyndman:08} or ARIMA models \citep{Box:13}, produce Gaussian predictive distributions. While sufficient for inventories of several thousand fast-selling items, Gaussian assumptions are grossly violated for the extremely large catalogues maintained by e-commerce platforms. There, demand is highly {\em intermittent} and {\em bursty}: long runs of zeros, with islands of high counts. Decision making requires quantiles of predictive distributions \citep{Snyder:11}, whose accuracy suffer under erroneous assumptions.

In this work, we detail a novel methodology for intermittent demand forecasting which operates in the industrial environment of a very large e-commerce platform. Implemented in {\tt Apache Spark} \citep{Zaharia:12}, our method is used to process many hundreds of thousands of items and several hundreds of millions of item-days. Key requirements are {\em automated parameter learning} (no expert interventions), {\em scalability} and a high degree of operational {\em robustness}. Our system produces forecasts both for short (one to three weeks) and longer lead times (up to several months), the latter require feature maps depending on holidays, sales days, promotions, and price changes.
Previous work on intermittent demand forecasting in Statistics is surveyed in \cite{Snyder:12}: none of these address longer lead times. On a modelling level, our proposal is related to \cite{Chapados:14}, yet several novelties are essential for operating at the industrial scale we target here.
This paper makes the following contributions:
\begin{itemize}
	\item
	A combination of generalized linear models and exponential smoothing time series models. The
	former enables medium and longer term forecasts, the latter provides temporal
	continuity and capture uncertainties appropriately over time. Compared to
	\cite{Chapados:14}, we provide empirical evidence for the usefulness of this
	combination.
	\item
	A novel algorithm for maximum likelihood parameter learning in state space
	models with non-Gaussian likelihood, using approximate Bayesian inference.
	While there is substantial related prior work, our proposal stands out in
	robustness and scalability. We show how approximate inference is solved by
	the Newton-Raphson algorithm, fully reduced to Kalman smoothing once per
	iteration. This reduction scales {\em linearly} (a vanilla implementation
	would scale cubically). While previously used in Statistics
	\cite[Sect.~10.7]{Durbin:12}, this reduction is not widely known in Machine
	Learning. If L-BFGS~\citep{Nocedal:06} is used instead (as proposed in \cite{Chapados:14}),
	approximate inference fails in our real-world use cases.
	\item
	A multi-stage likelihood, tailored to intermittent and bursty demand data
	(extension of \cite{Snyder:12}), and a novel transfer function for Poisson
	likelihood, which robustifies the Laplace approximation for bursty data. We
	demonstrate that our approach would not work without these novelties.
\end{itemize}

A preliminary version of the current work appeared in \cite{Seeger:16}.
In this paper, we present substantial modeling extensions as well as simplifications over \cite{Seeger:16} rendering the current work applicable to a wider class of time series data. In particular, we allow for the modeling of time series patterns (trend, seasonality) which can slowly drift over time. Such patterns occur frequently in real data. For example, overall sales may be larger on Sundays than on Mondays, but this ratio may change across the year. While such seasonality patterns are represented in \cite{Seeger:16} via features in the generalized linear model, they cannot drift, and only the average behavior across the entire time series can be learned.
We also provide detailed derivations here, omitted in \cite{Seeger:16}.

The structure of this paper is as follows. In \secref{glm}, we introduce intermittent demand likelihood function as well as a generalized linear model baseline. Our novel {\em latent state forecasting} methodology is detailed in \secref{latent}. \secref{latent-training} discusses the maximum likelihood parameter learning in our model using approximate Bayesian inference. We relate our approach to prior work in \secref{relwork}. In \secref{exper}, we evaluate our methods both on publicly available data and on a large dataset of real-world demand in the context of e-commerce, comparing against state of the art intermittent forecasting methods.

\section{Generalized Linear Models}\label{sec:glm}

In this section, we introduce a likelihood function for intermittent demand data, along with a generalized linear model as baseline. Denote demand by $z_{i t}\in\N$ ($i$ for item, $t$ for day). Our goal is to predict {\em distributions} of $z_{i t}$ in the future. We do this by fitting a probabilistic model to maximize the likelihood of training demand data, then drawing sample paths
from the fitted probabilistic model, which represent forecast distributions. In the sequel, we fix an item $i$ and write $z_t$ instead of $z_{i t}$.

A model is defined by a {\em likelihood} $P(z_t | y_t)$ and a latent function $y_t$. An example is the {\em Poisson}:
\begin{equation}\label{eq:lh-poisson}
P_{\text{poi}}(z | y) = \frac{1}{z!} \lambda(y)^{z}
e^{-\lambda(y)},\quad z\in\N,
\end{equation}
where the rate $\lambda(y)$ depends on $y$ through a transfer function. Demand data over large inventories is both intermittent (many $z_t = 0$) and bursty (occasional large $z_t$), and is not well represented by a Poisson. 

\begin{figure}[ht]
\begin{minipage}{0.48\textwidth}
\includegraphics[width=\textwidth]{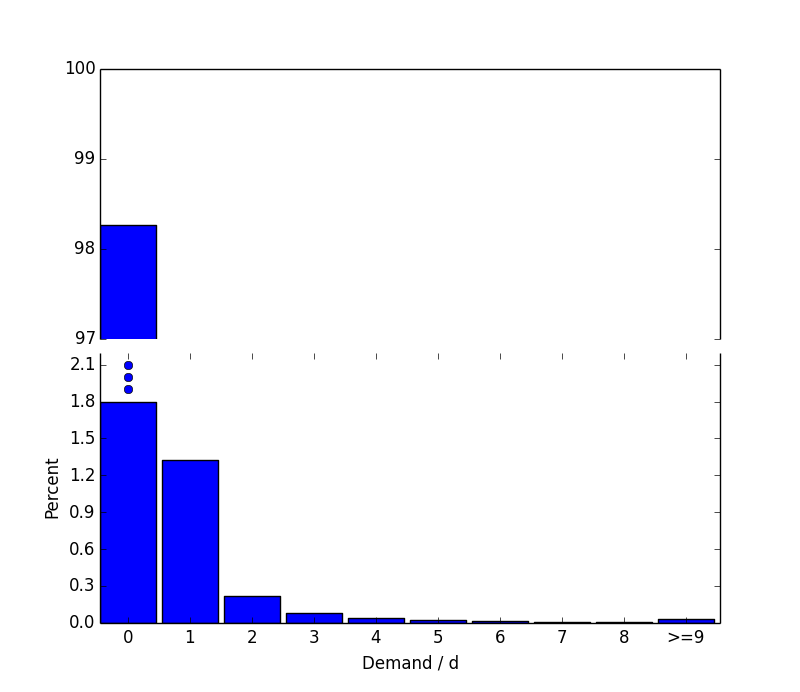}
\end{minipage}
\begin{minipage}{0.50\textwidth}
\includegraphics[width=\textwidth]{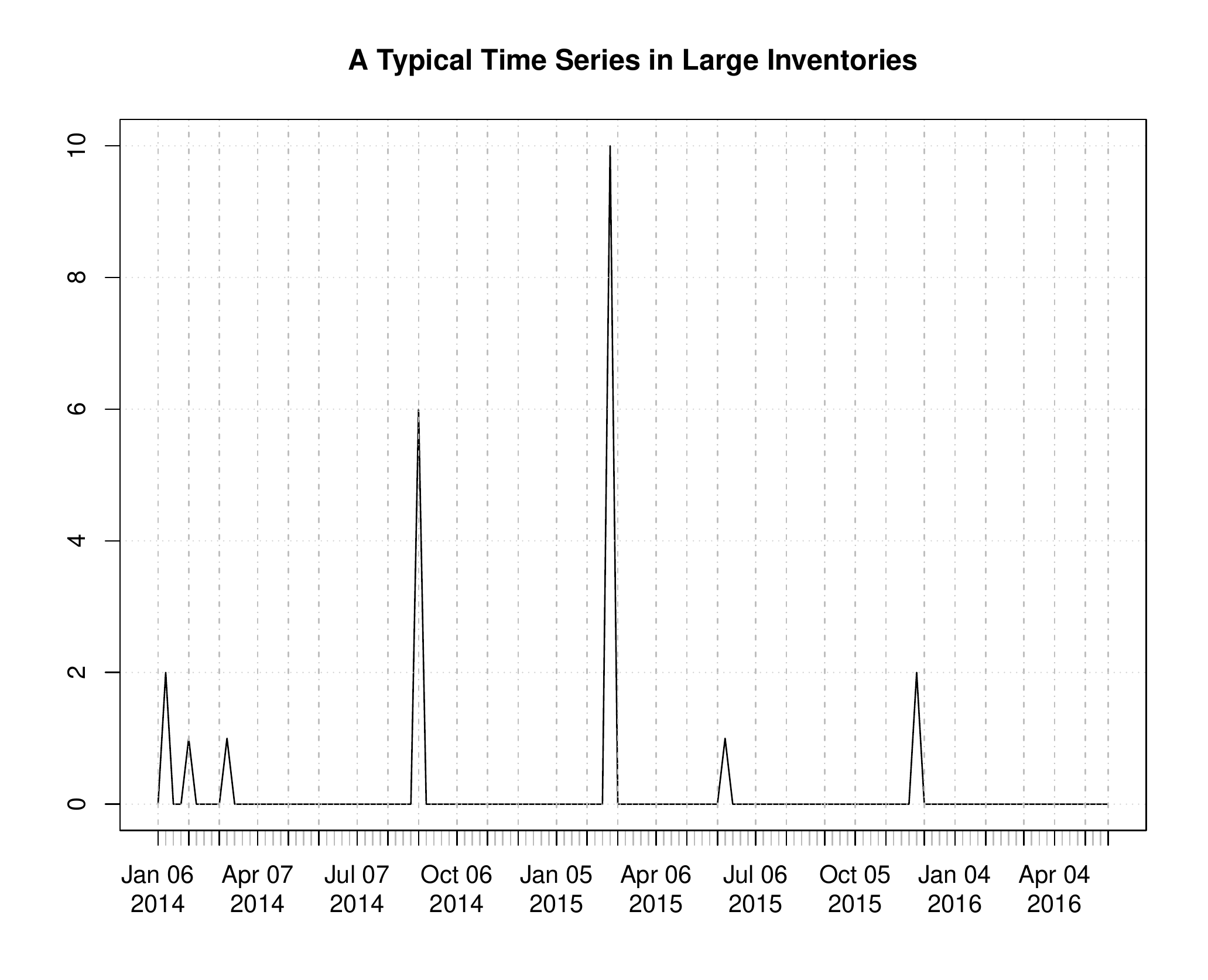}
\end{minipage}
\caption{\label{fig:demand-marginal}
  Left: Marginal histogram of demand values $\{z_{i t}\}$ for a typical dataset. Right: A typical time series in large inventories. }
\end{figure}

A better choice is the {\em multi-stage likelihood}, generalizing a proposal in \cite{Snyder:12}. This likelihood decomposes into $K=3$ stages, each with its latent function $y^{(k)}$. In stage $k=0$, we emit $z = 0$ with probability $\sigma(y^{(0)})$ where $\sigma(u) := (1 + e^{-u})^{-1}$ is the logistic sigmoid. Otherwise, we transfer to stage $k=1$, where $z = 1$ is emitted with probability $\sigma(y^{(1)})$. Finally, if $z\ge 2$, then stage $k=2$ draws $z-2$ from the Poisson \eqp{lh-poisson} with rate $\lambda(y^{(2)})$. In summary, the multi-stage likelihood reads as follows,
\begin{equation}\label{eq:lh-multi}
 P(z | \{y^{(k)}\}) = P_{\text{poi}}(z - 2 | y^{(2)})^{\Ind{z\ge 2}}
 \prod_{k=0}^1 \sigma(\tilde{z}_k y^{(k)})^{\Ind{z\ge k}},\quad
   \tilde{z}_k := \Ind{z=k} - \Ind{z>k}.
\end{equation}

If the latent function $y_t$ (or functions $y_t{}^{(k)}$) is linear, $y_t = \vx{t}^{\top}\vw{}$, we have a {\em generalized linear model} (GLM) \citep{McCullach:83}. Features in $\vx{t}$ include kernels anchored at holidays (Thanksgiving, Christmas, Halloween, etc.), seasonality indicators (DayOfWeek, WeekOfYear, MonthOfYear, etc.), promotion or price change indicators. The weights $\vw{}$ are learned by maximizing the training data likelihood. For the multi-stage likelihood, this amounts to separate instances of binary classification at stages 0, 1, and Poisson regression at stage 2.
%
Generalized linear forecasters work reasonably well, but have some important drawbacks: 
\begin{itemize}
\item They lack temporal continuity: for short term predictions, even simple smoothers can outperform a tuned GLM;
\item  More important, a GLM predicts overly narrow forecast distributions, whose widths do not grow over time, and it neglects temporal correlations.
\end{itemize}
 Both drawbacks are alleviated in Gaussian linear time series models, such as {\em exponential smoothing} (ES) \citep{Hyndman:08}. A major challenge is to combine this technology with general likelihood functions (Poisson, multi-stage) to enable intermittent demand forecasting.

\section{Latent State Forecasting}\label{sec:latent}

In this section, we develop latent state forecasting for intermittent demand, combining GLMs, general likelihoods, and exponential smoothing time series models. We begin with a single likelihood $P(z_t | y_t)$, for example the Poisson
\eqp{lh-poisson}, then consider a multi-stage extension. The latent process is
\begin{equation}\label{eq:issm}
y_t = \va{t}^{\top}\vl{t-1} + b_t,\quad b_t = \vw{}^{\top}\vx{t},\quad
\vl{t} = \mxf{}\vl{t-1} + \vg{t}\eps_t,\quad \eps_t\sim N(0,1).
\end{equation}
Here, $b_t$ is the GLM deterministic linear function, $\vl{t}$ is a {\em latent state}. This {\em innovation state space model} (ISSM) \citep{Hyndman:08} is defined by $\va{t}$, $\vg{t}$ and $\mxf{}$, as well as the prior $\vl{0}\sim P(\vl{0})$. Note that ISSMs are characterized by a {\em single} Gaussian innovation variable $\eps_t$ per time step. 
Here, the innovation vector $\vg{t}$ comes in terms of parameters to be learned (the innovation strengths), while $\mxf{}$ and $\va{t}$ are fixed\footnote{
	They may still contain {\em hyperparameters}, but these have to be adjusted by
	the user: they are not learned by maximum likelihood.}.
Moreover, the initial state $\vl{0}$ has to be specified. We do so by specifying a Gaussian prior distribution $P(\vl{0})$, whose parameters (means, standard deviation) are learned from data as well.

The proposed model differs from the exponential smoothing (\cite{Hyndman:08}) in that there are two innovations per time step: $y_t \rightarrow z_t$ and $\epsilon_t$. The latter is used for propagating the latent state only. Moreover, an important technical contribution here is combining exponential smoothing with general likelihood functions (Poisson, multi-stage).

The parameters to be learned by maximum likelihood include the weights $\vw{}$ of the linear part $b_t$, parameters in $\vg{t}$, and parameters of the prior $P(\vl{0})$. In the sequel, we collect these parameters in the vector $\vth{}$, whose precise size and encoding depends on the structure of the model.
Note that, for simplicity, we treated the weights $\vw{}$ as a part of the parameters to be estimated in the main text of the paper.
In Appendix~\ref{app:InfOverW}, we discuss an extension of this method where we do inference over the weights $\vw{}$, instead of estimating them as part of $\vth{}$.
We further generalize the method to the missing observations case in Appendix~\ref{app:MissingObs}. 

\subsection{Composing ISS Models}

Real forecasting data often exhibits multiple structural patterns, such as a level, a trend, and various seasonality patterns at different resolutions. Our framework allows one to compose basic ISSMs additively, in order to match such complexities in the data. Basic components are detailed shortly. Their combination works by stacking their latent state vectors on top of each other. Given component ISSMs $(\mxf[j]{}, \vg[j]{}, \va[j]{})$, $j=1,2,\ldots, K$, we construct a composite ISSM by making $\mxf{}$ the block-diagonal matrix with blocks $\mxf[j]{}$, and $\va{t}$, $\vg{t}$ by stacking the component vectors respectively. 

\subsection{Modeling Level}

The simplest possible ISSM maintains a {\em level} component only. Our previous work \cite{Seeger:16} considered only this case. This is referred to here as \texttt{LevelISSM}:
\[
\va{t} = [ 1 ],\quad \mxf{} = [ 1 ],\quad \vg{t} = [ \alpha ],\quad \alpha>0.
\]
The level $l_t \in \R$ evolves over time by adding a random innovation $\alpha \eps_t\sim \Nc(0,\alpha^2)$ to the previous level, so that $\alpha$ specifies the amount of level drift over time. At time $t$, the previous level $l_{t-1}$ is used as prediction $u_t$ and then the level is updated. The initial state prior $P(l_0)$ is given by $l_0 \sim \Nc(\mu_0, \sigma_0^2)$. For \texttt{LevelISSM}, we learn the parameters $\alpha>0$, $\mu_0$, $\sigma_0>0$ apart from the weights $\vw{}$ of the deterministic part.

\subsection{Modeling Level and Linear Trend}

A piecewise linear random process is modeled by using a two-dimensional latent state $\vl{t}\in\R^2$, where one dimension represents the level and the other represents the slope \citep{Hyndman:08}. Such a \texttt{LevelTrendISSM} is given by
\[
\va{t} = \left[\begin{array}{c}
1 \\
1
\end{array}\right], \quad \mxf{} = \left[\begin{array}{cc}
1 & 1 \\
0 & 1
\end{array}\right], \quad \vg{t} = \left[\begin{array}{c}
\alpha \\
\beta
\end{array}\right],
\]
where $\alpha>0$, $\beta>0$.
Both the level and slope components evolve over time by adding innovations $\alpha \eps_t$ and $\beta \eps_t$ respectively, so that $\beta>0$ is the innovation strength for the slope. The level at time $t$ is the sum of level at $t-1$ and slope at $t-1$ (linear prediction). The initial state prior $P(\vl{0})$ is given by $\vl{0}\sim \Nc(\vmu{0},\diag(\vsigma{0}^2))$. We learn the parameters $\alpha>0$, $\beta>0$, $\vmu{0}$, $\vsigma{0} > \vzero$. Finally, we also support damping level and slope factors, by replacing the 1s in $\mxf{}$ and $\va{t}$ by hyperparameters $<1$.

\subsection{Modeling Seasonality}

The ISSM components described so far assume that the latent state at time $t-1$ is a good predictor of what happens at time $t$. Time series often exhibit seasonality patterns, which repeat with a known periodicity. In that setting, the prediction at time $t$ depends on what happened during the last season (or cycle). For example, in ``day-of-week'' seasonality, the prediction at day $t$ depends on the level at $t-7$ (a week ago). 

The same time series often has more than one seasonality pattern. For example, if the targets denote the weekly sales of an item sold, the time series may exhibit ``week/month/quarter-of-year'' seasonality patterns, as well as seasonality patterns centered around moving holidays such as Thanksgiving. Our framework can model any of these seasonality patterns, including user-defined ones, as well as any combination of them.

Each seasonality pattern can be described by a set of seasonal factors (or seasons) associated with it. For example, in the day-of-week pattern there are seven factors, one for each day of the week, and $\vl{t} \in \R^{7}$. Seasonality patterns are represented as \texttt{SeasonalityISSM}. For day-of-week, we have
\[
\va{t} = \mathds{1}_{\{\text{day}(t) = j\}}, \quad
\mxf{} = \Id, \quad
\vg{t} = \gamma \tvg{j}.
\]
Here $\va{t}$ is an indicator vector specifying \textit{when a factor is used}.
Our definition of $\va{t}$ implies that for any given $t$, exactly one factor is used. The definition of $\vg{t}$ and $\mxf{}$ determines how the factors evolve over time. Since factors occupy fixed entries in $\vl{t}$, we can use $\mxf{} = \Id$. Then, a certain number of them are perturbed by a multiple of the innovation variable $\eps_t$, where $\gamma>0$ is the innovation strength to be learned. The definition of $\tvg{t}$ depends on how many times a factor is used per cycle.
If the granularity of the data is daily, then each factor of the day-of-week seasonality is used exactly once, and we simply have $\tvg{t} = \va{t}$. In this case, a factor $j$ is used ($a_{t j} = 1$), and then updated immediately afterwards ($g_{t j} = \gamma > 0$).

However, if the data granularity is hourly, each factor in day-of-week is used 24 times per cycle, during all hours of its day. In this case, we define $\tilde{g}_{t j} = (1/N_j) a_{t j}$, where $N_j$ is the number of times factor $j$ is used during a cycle. Namely, each factor is associated a ``unit budget'' per cycle for its update, which is spread uniformly over all usage times. While in the day-of-week example, all $N_j = 24$, these weights can be different if grouping is used (see below), or also when leap hours are encountered (switching between summer and winter times). Finally, the initial state prior $P(\vl{0})$ is given by $\vl{0} \sim \Nc(\vmu{0}, \sigma_0^2\Id)$, where $\vmu{0}$ is the vector of means, and $\sigma_0>0$ is the common standard deviation. Learnable parameters for \texttt{SeasonalityISSM} are $\gamma>0$, $\vmu{0}$, and $\sigma_0>0$.

\paragraph{Grouping of seasonality factors}
When using ``hour-of-day'' or ``day-of-month'' seasonality, we can end up with a large number of factors. In such situations, it can be beneficial to {\em tie} some of these factors to be the same. In hour-of-day, all the night hours might have same pattern, while each morning hour could show different behavior. Or in day-of-week, we may only want to distinguish between workdays and weekend days.
Here we allow the user to simply group seasonality factors, by specifying a mapping $\pi(j)$ from atomic factors to group factors. The latent state $\vl{t}$ contains one factor per group (indexed by $h$ rather than $j$), and we have that $a_{t h} = \mathds{1}_{\pi(\text{time}(t)) = h}$. Moreover, we employ ``budget weighting'' by defining $\tilde{g}_{t h} = (1 / N_h) a_{t h}$, where $N_h$ is the number of times the factor for group $h$ is used during the cycle. As an example, consider day-of-week, where Monday is the first atomic factor. The workday versus weekend day grouping is given by $\vpi{} = [0, 0, 0, 0, 0, 1, 2]$: all workdays are lumped together, while Saturday and Sunday remain on their own. Also, $N_0 = 5$, $N_1 = N_2 = 1$ in this case. While grouping requires model selection by the user, it can have a number of benefits. First, fewer seasonality factors reduce the risk of overfitting. Second, since inference in our model scales cubically with the dimensionality of $\vl{t}$, grouping can speed up computations significantly.

\paragraph{Modeling custom/user-defined seasonality}
It is straightforward to represent any user-defined seasonality pattern in our framework.
To implement this, one just needs to provide the definitions of the ``sampling'' vectors $\va{t}$ as well as the ``update'' vectors $\tvg{t}$. It is important to note that subsequent cycles can even differ in their lengths. An example based on custom seasonality pattern is discussed in \secref{custom-seasonality}.

\subsection{Handling Missing Observations}
In real-world retail forecasting, it is common for an item to go out of stock and its sales $z_t$ only reveals a fraction of the true demand. The probabilistic nature of latent state forecasting makes it easy to use the out of stock information. Whenever an item is out of stock, the data $z_t = 0$ is explained away and we drop the corresponding likelihood. Note that an item may be partially out of stock during a day, still creating some sales, i.e., $\rho_t\in (0, 1)$ percentage of the time an item is available, one simply needs to rescale the likelihood by taking the power to $\rho_t.$ We could treat $z_t$ as unobserved, but lower-bounded by the sales, and an expectation maximization extension may be applied, which is beyond the scope of the current paper. Details can be found in Section~\ref{sec:exper-oos} and Appendix~\ref{app:MissingObs}.



\section{Training. Prediction. Multiple Stages}
\label{sec:latent-training}

\begin{algorithm}
	\caption{Finding mode of the complete likelihood:
		$\argmin_{\veps, \vl{0}} F(\veps{},\vl{0})$, where $F(\veps{},\vl{0}) = -\log P(\vz{}|\vy{}) P(\veps{}) P(\vl{0}),\ \vy{} =
		\vy{}(\veps{},\vl{0})$ and $\veps = [\eps_1,\dots, \eps_{T-1}]$.}
	\begin{algorithmic}[1]  	
	\label{alg:Mode}			
		\WHILE{NOT CONVERGED} 
		\STATE Compute $\vy[i]{} = \vy{}(\veps[i]{}, \vl[i]{0})$ via the forward pass; 
		\STATE Compute the second order Taylor approximation: $(y_t^{(i)}, z_t)\mapsto (\tsz{t},
		\sigma_t^2)$, leading to the quadratic surrogate
		\[
		\tilde{F}(\veps{},\vl{0}) = -\log N(\tvz{} | \vy{}, \diag[\sigma_t^2])
		P(\veps{}) P(\vl{0}),\quad \vy{} = \vy{}(\veps{},\vl{0}).
		\]
		\STATE Find minimum point of $\tilde{F}$:
		\[
		[\veps[i+1]{}, \vl[i+1]{0}] = \argmin \tilde{F}(\veps{},\vl{0}).
		\]
		\ENDWHILE
	\end{algorithmic}
\end{algorithm}

We would like to learn $\vth{}$ by maximizing the likelihood of data $[z_t]_{t=\rng{T}}$. Compared to the GLM case, this is harder to do, since latent (unobserved) variables $\vs{} = [\eps_1,\dots, \eps_{T-1}, \vl{0}{}^{\top}]^{\top}$ have to be integrated out. If our likelihood $P(z_t | y_t)$ was Gaussian, this marginalization could be computed analytically via Kalman smoothing \citep{Hyndman:08}. 

\begin{figure}[!htb]
\centering
\begin{minipage}{.32\textwidth}
\includegraphics[width=\textwidth]{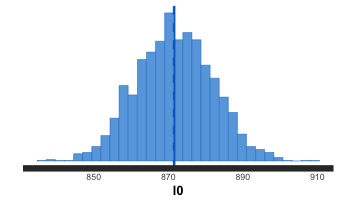}
\end{minipage}
\begin{minipage}{.32\textwidth}
\includegraphics[width=\textwidth]{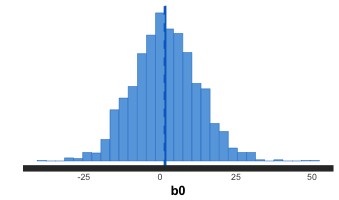}
\end{minipage}
\begin{minipage}{.32\textwidth}
\includegraphics[width=\textwidth]{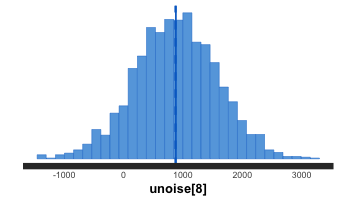}
\end{minipage}
\caption{The true posterior distribution of a subset of latent variables for Poisson likelihood using Markov Chain Monte Carlo (MCMC). The ``bell'' shaped posteriors for the latent variables illustrates our motivation of using Laplace approximation.}
\label{fig:posterior}
\end{figure}

With a non-Gaussian likelihood, the problem is analytically intractable, yet is amenable to the {\em Laplace approximation} \citep[Sect.~4.4]{Bishop:06}, which approximates the posterior distribution of the latent variables with Gaussian distribution through the second-order Taylor expansion. Figure~\ref{fig:posterior} shows that, for the Poisson likelihood, the motivation of using Gaussian distribution to approximate the posterior of the latent variables $l_0, b_0$ and one of the innovations $\epsilon_t.$

More precisely, the exact log likelihood is 
\[
\log P(\vz{}|\vth{}) = \log\int P(\vz{},\vs{}|\vth{})\, d\vs{} = \log\int \prod_t P(z_t | y_t) P(\vs{})\, d\vs{},
\] where $\vy{}=\vy{}(\vs{})$ is the affine mapping given by \eqp{issm}. We proceed in two steps. 

\begin{itemize}
\item First, we find the {\em mode} of the posterior: $\hvs{} = \argmax \log P(\vz{},\vs{}|\vth{})$, the {\em inner optimization} problem.
\item Second, we replace $-\log P(\vz{},\vs{} | \vth{})$ by its quadratic Taylor approximation $f(\vs{}; \vth{})$ at the mode.
\end{itemize}
The criterion to replace the negative log likelihood is $\psi(\vth{}) := -\log\int e^{-f(\vs{}; \vth{})}\, d\vs{}$. More precisely, denote $\phi_t(y_t) := -\log P(z_t | y_t)$, and let $\hvy{} = \vy{}(\hvs{})$, where $\hvs{}$ is the posterior mode. The log-concavity of the likelihood implies that $\phi_t(y_t)$ is convex, and $\phi_t''(y_t) > 0$. The quadratic Taylor approximation to $\phi_t(y_t)$ at $\hsy{t}$ is $\tilde{\phi}_t(y_t) := -\log N(\tsz{t} | y_t,\sigma_t^2)$, where $\sigma_t^2 = 1/\phi_t''(\hsy{t})$ and $\tsz{t} = \hsy{t} - \sigma_t^2\phi_t'(\hsy{t})$. Now, Laplace's approximation to $-\log P(\vz{}|\vth{})$ can be written as
\begin{equation}\label{eq:outer-crit}
\psi(\vth{}) = -\log\int\prod_t N(\tsz{t} | y_t,\sigma_t^2) P(\vs{})\, d\vs{}
+ \sum_t\left( \phi_t(\hsy{t}) - \tilde{\phi}_t(\hsy{t}) \right),\quad
\vy{} = \vy{}(\vs{}; \vth{}).
\end{equation}
For log-concave\footnote{
	Unless otherwise said, all likelihoods in this paper are log-concave.}
$P(z_t | y_t)$, the inner optimization is a convex problem. We use the {\em Newton-Raphson} algorithm to solve it (see Algorithm \ref{alg:Mode}). This algorithm iterates between fitting the current criterion by its local second order approximation and minimizing the quadratic surrogate. For the former step, we compute $y_t$ values by a forward pass \eqp{issm}, then replace the potentials $P(z_t | y_t)$ by $N(\tilde{z}_t | y_t,\sigma_t^2)$, where the values $\tilde{z}_t$, $\sigma_t^2$ are determined by the second order fit (as above, but $\hsy{t}\to y_t$). The latter step amounts to computing the posterior mean  $\Ex[\vs{}]$ (equal to the mode) of the resulting linear-Gaussian model. This inference problem is solved by a variant of Kalman smoothing.
The technical details for this posterior inference problem are given in Appendix~\ref{app:Inference}. In practice, for numerical stability, we solve this inference problem by the information square root filter (see Appendix~\ref{app:ISF}).\footnote{
	We use a numerically robust implementation of Kalman smoothing, detailed in \citep[Sect.~12]{Hyndman:08}.}

Not only finding the mode $\hvs{}$, but also the computation of $\nabla_{\vth{}}\psi$, is fully reduced to Kalman smoothing. This point is crucial. We can find $\hvs{}$ by the most effective optimization algorithm there is. In general, each Newton step requires the $O(T^3)$ inversion of a Hessian matrix. We reduce it to Kalman smoothing, a robust algorithm with $O(T)$ scaling. As shown in \secref{relwork}, Newton-Raphson is essential here: other commonly used optimizers fail to find $\hvs{}$ in a reasonable time.

Prediction samples are obtained as follows. Denote observed demand by $[z_{1}, z_{2},\dots, z_{T}]$, unobserved demand in the prediction range by $[z_{T+1}, z_{T+2}, \dots]$. We run Newton-Raphson one more time to obtain the Gaussian approximation to the posterior $P(\vl{T} | \vz{1:T})$ over the final state. For each sample path $[z_{T+t}]$, we draw $\vl{T}\sim P(\vl{T} | \vz{1:T})$, $\eps_{T+t}\sim N(0,1)$, compute $[y_{T+t}]$ by a forward pass, and $z_{T+t}\sim P(z_{T+t} | y_{T+t})$. Drawing prediction samples is not more expensive than from a GLM.

Finally, we generalize latent state forecasting to the multi-stage likelihood. As for the GLM, we learn parameters $\vth[k]{}$ separately for each stage $k$. Stages $k=0, 1$ are binary classification, while stage $k=2$ is count regression. Say that a day $t$ is {\em active} at stage $k$ if $z_t\ge k$. Recall that with GLMs, we simply drop non-active days. Here, we use ISSMs for $[y_t^{(k)}]$ on the full range $t=\rng{T}$, but all non-active $y_t^{(k)}$ are considered {\em unobserved}: no likelihood potential is associated with $t$. Both Kalman smoothing and mode finding (Laplace approximation) are adapted to missing observations, which presents no difficulties (see also \secref{exper-oos}).

\subsection{Some Details}\label{sec:latent-details}

In this section, we sketch technical details, most of which are novel contributions. As demonstrated in our experiments, these details are essential for the whole approach to work robustly at the intended scale on our difficult real-world data. Full details are given in a supplemental report.

We use L-BFGS \citep{Nocedal:06} for the {\em outer optimization} of $\psi(\vth{})$, encoding the constrained parameters: $\alpha = \alpha_m + (\alpha_M - \alpha_m)\sigma(\theta_1)$; $0 < \alpha_m < \alpha_M$; $\sigma_0 = \log(1 + e^{\theta_2}) > 0$. We add a quadratic regularizer $\sum_j(\rho_j/2) (\theta_j - \bar{\theta}_j)^2$ to the criterion, where $\rho_j$, $\bar{\theta}_j$ are shared across all items. Finally, recall that with the multi-stage likelihood, day $t$ is unobserved at stage $k>1$ if $z_t<k$. If for some item, there are less than 7 observed days in a stage, we skip training and {\em fall back} to fixed parameters $\bvth{}$.

Every single evaluation of $\psi(\vth{})$ requires finding the posterior mode $\hvs{}$. This high-dimensional inner optimization has to converge robustly in few iterations: $\hvs{} = \argmin F(\vs{}; \vth{}) := -\log P(\vz{} | \vs{}) - \log P(\vs{}) = \sum_t\phi_t(y_t) - \log P(\vs{})$. The use of Newton-Raphson and its reduction to linear-time Kalman smoothing was noted above. The algorithm is extended by a line search procedure as well as a heuristic to pick a starting point $\vs{0}$ (Appendix~\ref{app:LineSearch}).

We have to compute the gradient $\nabla_{\vth{}}\psi(\vth{})$, where the criterion is given by \eqp{outer-crit}. The main difficulty here are indirect dependencies: $\psi(\vth{}, \hvy{}, \hvs{})$, where $\hvy{} = \vy{}(\hvs{};\vth{})$, $\hvs{} = \hvs{}(\vth{})$. Since $\hvs{}$ is computed by an iterative algorithm, commonly used automated differentiation tools do not sensibly apply here. Maybe the most difficult indirect term is $(\partial_{\hvs{}}\psi)^{\top} (\partial\hvs{}/\partial\theta_j)$, where $\theta_j\in\vth{}$. First, $\hvs{}$ is defined by $\partial_{\hvs{}} F = \vzero$. Taking the derivative w.r.t.\ $\theta_j$ on both sides, we obtain $(\partial\hvs{}/\partial\theta_j) = - (\partial_{\hvs{},\hvs{}} F)^{-1} \partial_{\hvs{},\theta_j} F$, so we are looking at $-(\partial_{\hvs{},\theta_j} F)^{\top} (\partial_{\hvs{},\hvs{}} F)^{-1} (\partial_{\hvs{}}\psi)$. It is of course out of the question to compute and invert $\partial_{\hvs{},\hvs{}} F$. But $(\partial_{\hvs{},\hvs{}} F)^{-1} (\partial_{\hvs{}}\psi)$ corresponds to the {\em posterior mean} for an ISSM with Gaussian likelihood, which depends on $\partial_{\hvs{}}\psi$. This means that the indirect gradient part costs {\em one more} run of Kalman smoothing, independent of the number of parameters $\theta_j$. Note that the same reasoning underlies our reduction of Newton-Raphson to Kalman smoothing.
We refer the reader to Appendix~\ref{app:Gradient} for details on the gradient computation.

A final novel contribution is essential for making the Laplace approximation work on real-world bursty demand data. Recall the transfer function $\lambda(y)$ for the Poisson likelihood \eqp{lh-poisson} at the highest stage $k=2$. As shown in \secref{relwork}, the exponential choice $\lambda = e^y$ fails for all but short term forecasts. With a GLM, the logistic transfer $\lambda(y) = g(y)$ works well, where $g(u) := \log( 1 + e^u )$. It behaves like $e^y$ for $y<0$, but grows linearly for positive $y$. However, it exhibits grave problems for latent state forecasting. Denote $\phi(y) := -\log P(z | y)$, where $P(z | y)$ is the Poisson with logistic transfer. Recall Laplace's approximation from \secref{latent-training}: $\phi(\cdot)$ is fit by a quadratic $\tilde{\phi}(\cdot) = (\cdot - \tsz{})/(2\sigma^2)$, where $\sigma^2 = 1/\phi''(y)$, $\tsz{} = y - \sigma^2\phi'(y)$. For large $y$ {\em and} $z = 0$, these two terms scale as $e^y$, while for $z>0$ they grow polynomially. In real-world data, we regularly exhibit sizable counts (say, a few $z_t > 25$, driving up $y_t$), followed by a single $z_t = 0$. At this point, huge values $(\tilde{z}_t, \sigma_t^2)$ arise, causing cancellation errors in $\psi(\vth{})$, and the outer optimization terminates prematurely.

The root cause for these issues lies with the transfer function: $g(y)\approx y$ for large $y$, its curvature behaves as $e^{-y}$. Our remedy is to propose the novel {\em twice logistic transfer function}: $\lambda(y) = g( y ( 1 + \kappa g(y) )$, where $\kappa>0$. If $\phi^{\kappa}(y) = -\log P(z | y)$ with the new transfer function, then $\phi^{\kappa}(y)$ behaves similar to $\phi^0(y)$ for small or negative $y$, but crucially $(\phi^{\kappa})''(y)\approx 2\kappa$ for large $y$ {\em and any} $z\in\N$. This means that Laplace approximation terms are $O(1/\kappa)$. Setting $\kappa = 0.01$ resolves all problems described above. Importantly, the resulting Poisson likelihood is log-concave for any $\kappa\ge 0$. We conjecture that similar problems may arise with other ``local'' variational or expectation propagation inference approximations as well. The twice logistic transfer function should therefore be of wider applicability.

\section{Related Work}\label{sec:relwork}

Our work has precursors both in Statistics and Machine Learning. Maximum likelihood learning for exponential smoothing  models \citep{Brown:59} is developed in \cite{Hyndman:08}. These methods are limited to Gaussian likelihood, approximate Bayesian inference is not used. Starting from Croston's method \citep[Sect.~16.2]{Hyndman:08}, there is a sizable literature on intermittent demand forecasting, as reviewed in \cite{Snyder:12}. The best-performing method in \cite{Snyder:12} uses negative binomial likelihood and a damped dynamic, parameters are learned by maximum likelihood. There is no latent (random) state, and neither non-Gaussian inference nor Kalman smoothing are required. It does not allow for a combination with GLMs.

We employ approximate Bayesian inference in a linear dynamical system, for which there is a lot of prior work in Machine Learning \citep{Beal:03,Barber:06,Barber:11}. While Laplace's technique is the most frequently used deterministic approximation in Statistics, both in publications and in automated inference systems \citep{Rue:09}, other techniques such as expectation propagation are applicable to models of interest here \citep{Minka:01a,Heskes:02}. The robustness and predictable running time of Laplace's approximation are key in our application, where inference is driving parameter learning, running in parallel over hundreds of thousands of items. Expectation propagation is not guaranteed to converge, and Markov chain Monte Carlo methods even lack automated convergence tests.

The work most closely related to ours is \citep{Chapados:14}. They target intermittent demand forecasting, using a Laplace approximation for maximum likelihood learning, allow for a combination with GLMs, and go beyond our work transferring information between items by way of a hierarchical prior distribution. Their work is evaluated on small datasets and short term scenarios only. In contrast, our system runs robustly on many hundreds of thousands of items and many millions of item-days, a three orders of magnitude larger scale than what they report. They do not explore the value of a feature-based deterministic part, which on our real-world data is essential for medium term forecasts. We find that a number of choices in \citep{Chapados:14} are limiting when it comes to robustness and scalability. First, they choose a likelihood which is not log-concave for two reasons: they use a negative binomial distribution instead of a Poisson, and they use zero-inflation instead of a multi-stage setup.\footnote{
	Zero-inflation, $p_0\Ind{z_t=0} + (1-p_0) P'(z_t | y_t)$, destroys log-concavity for $z_t = 0$.}
This means their inner optimization problem is non-convex, jeopardizing robustness and efficiency of the nested learning process. Moreover, in our multi-stage setup, the {\em conditional probability} of $z_t=0$ versus $z_t>0$ is represented exactly, while zero-inflation caters for a time-independent zero probability only.

Next, they use an exponential transfer function $\lambda = e^y$ for the negative binomial rate, while we propose the novel twice logistic function (\secref{latent-details}). Experiments with the exponential choice on our data resulted in total failure, at least beyond short term forecasts. Its huge curvature for large $y$ results in extremely large and instable predictions around holidays. In fact, the exponential function causes rapid growth of predictions even without a linear function extension, unless the random process is strongly damped.
Finally, they use a standard L-BFGS solver for their inner problem, evaluating the criterion using additional sparse matrix software. In contrast, we enable Newton-Raphson by reducing it to Kalman smoothing. 
\begin{wrapfigure}{l}{0.5\textwidth}
	\begin{center}
		\includegraphics[width=0.48\textwidth]{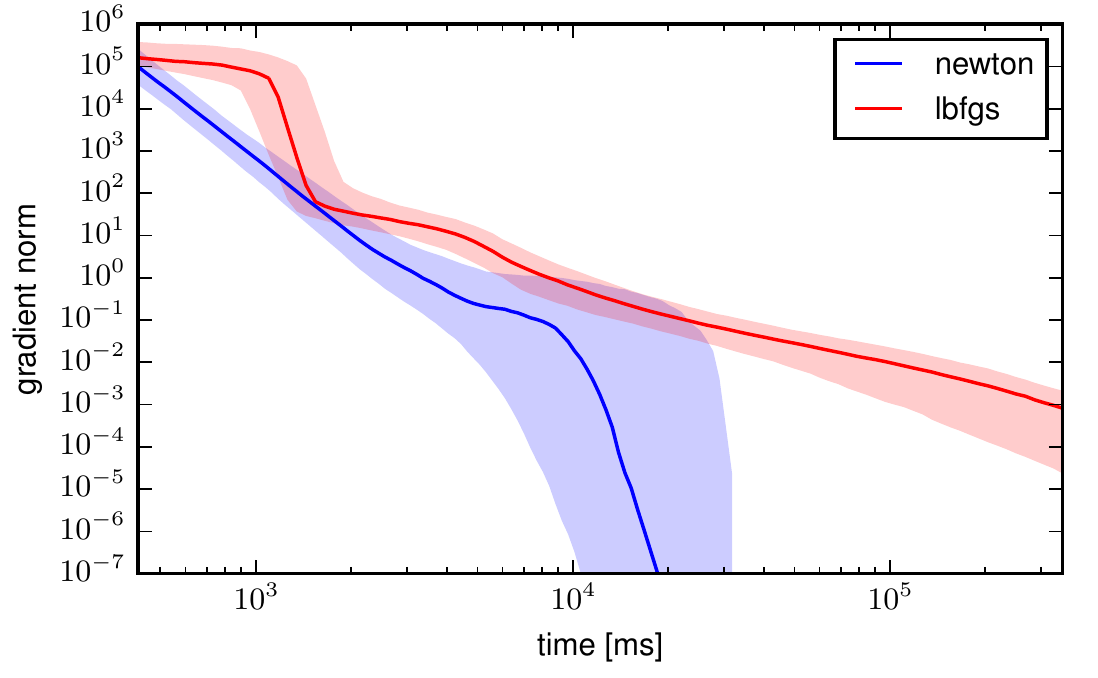}
	\end{center}
	\caption{\label{fig:optimizer} Comparison Newton-Raphson vs.\ L-BFGS for inner optimization. Sampled at first evaluation of $\psi(\vth{})$. Shown are median (p10, p90) over ca.~1500 items. L-BFGS fails to converge to decent accuracy.}
	\vspace{-15mm}
\end{wrapfigure}

In \figref{optimizer}, we evaluate the usefulness of L-BFGS for mode finding in our setup.\footnote{
	The inner problem is convex, its criterion is efficiently implemented (no dependence on foreign code). The situation in \cite{Chapados:14} is likely more difficult.}
L-BFGS clearly fails to attain decent accuracy in any reasonable amount of time, while Newton-Raphson converges reliably. Such inner reliability is key to reaching our goal of fully automated learning in an industrial system. In conclusion, while the lack of public code for \citep{Chapados:14} precludes a direct comparison, their approach, while partly more advanced, should be limited to smaller problems, shorter forecast horizons, and would be hard to run in an industrial setting.

\section{Experiments}\label{sec:exper}
In this section, we present experimental results, comparing variants of our approach to related work.

\subsection{Out of Stock Treatment}\label{sec:exper-oos}

With a large and growing inventory, a fraction of items is {\em out of stock} at any given time, meaning that order fulfillments are delayed or do not happen at all. When out of stock, an item cannot be sold ($z_t = 0$), yet may still elicit considerable customer demand. The probabilistic nature of latent state forecasting renders it easy to use out of stock information. If an item is not in stock at day $t$, the data $z_t=0$ is explained away, and the corresponding likelihood term should be dropped. As noted in \secref{latent-training}, this presents no difficulty in our framework.

\begin{figure}[ht!]
  \centering
  \begin{tabular}{cc}
    \includegraphics[width=0.48\textwidth]{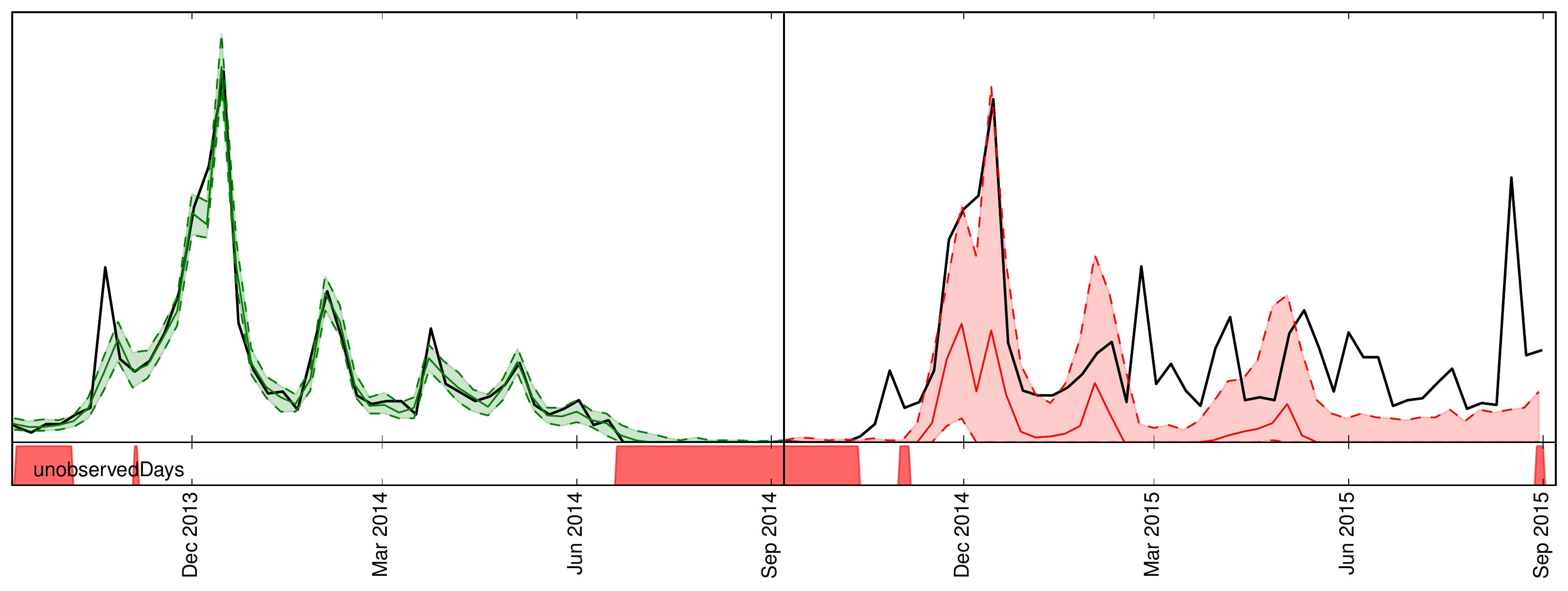} &
    \includegraphics[width=0.48\textwidth]{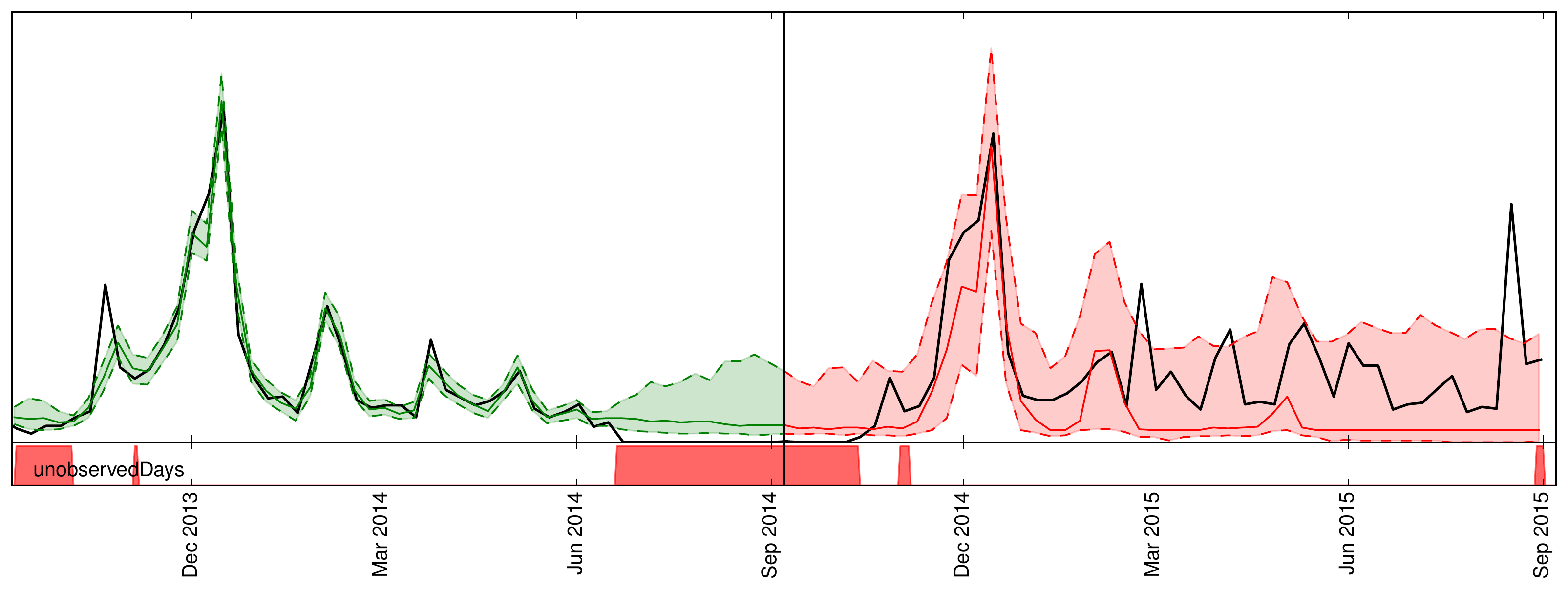}
  \end{tabular}
  \caption{\label{fig:outofstock}
    Demand forecast for an item which is partially out of stock. Each panel: Training range left (green), prediction range right (red), true targets black. In color: Median, P10 to P90. Bottom: Out of stock ($\ge 80\%$ of day) marked in red. {\bf Left}: Out of stock signal ignored. Demand forecast drops to zero, strong underbias in prediction range. {\bf Right}: Out of stock regions treated as missing observations. Demand becomes uncertain in out of stock region. No underbias in prediction range.}
\end{figure}

In \figref{outofstock}, we show demand forecasts for an item which is out of stock during certain periods in the training range. It is obvious that ignoring the out of stock signal leads to systematic underbias (since $z_t=0$ is interpreted as ``no demand''). This underbias is corrected for by treating out of stock regions as having unobserved targets.


%
In the rest of this section, latent state forecasting is taking out of stock information into account.

\subsection{Comparative Study}

We present experimental results obtained on a number of datasets, containing intermittent counts time series. \parts{} contains monthly demand of spare parts at a US automobile company, is publicly available, and was previously used in \citep{Hyndman:08,Snyder:12,Chapados:14}. Further results are obtained on internal daily e-commerce sales data. In either case, we subsampled the sets in a stratified manner from a larger volume used in our production setting. \articlesSubset{} is medium size and contains fast and medium moving items. \articlesAll{} is a large dataset (more than 500K items, 150M item-days), being the union of \articlesSubset{} with items which are slower moving. Properties of these datasets are given in \figref{stats-and-curves}, top left.
Demand is highly intermittent and bursty in all cases, as witnessed by a large $CV^2$ and a high proportion of $z_t=0$: these properties are typical for supply chain data. Not only is \articlesAll{} much larger than any public demand forecasting dataset we are aware of, our internal datasets consists of longer series (up to $10\times$) and are more bursty than \parts{}.

The following methods are compared. \ets{} is exponential smoothing with Gaussian additive errors and automatic model selection, a frequently used R package \citep{Hyndman:08a}. \snyder{} is our implementation of the negative binomial damped dynamic variant of \cite{Snyder:12}. We consider two variants of our latent state forecaster: \issm{} without features, and \issmfeatures{} with a feature vector $\vx{t}$ (basic seasonality, kernels at holidays, price changes, out of stock). Predictive distributions are represented by 100 samples over the prediction range (length 8 for \parts{}, length 365 for others). We employ quadratic regularization for all methods except \ets{} (see \secref{latent-details}). Hyperparameters consist of regularization constants $\rho_j$ and centers $\bar{\theta}_j$ (full details are given in the supplemental report). We tune\footnote{
  We found that careful hyperparameter tuning is important for obtaining good results, also for \snyder{}. In contrast, regularization is not even mentioned in \cite{Snyder:12} (our implementation of \snyder{} includes the same quadratic regularization as for our methods).}
such parameters on random $10\%$ of the data, evaluating test results on the remaining $90\%$. For \issm{} and \issmfeatures{}, we use two sets of tuned hyperparameters on the largest set \articlesAll{}: one for the \articlesSubset{} part, the other for the rest.

Our metrics quantify the forecast accuracy of certain quantiles of predictive distributions. They are defined in terms of {\em spans} $[L, L+S)$ in the prediction range, where $L$ are {\em lead times}. In general, we ignore days when items are out of stock (see \figref{stats-and-curves}, top left, for in-stock ratios). If $\pi_{i t} = \Ind{i\, \text{in stock at}\, t}$, define $Z_{i; (L,S)} = \sum_{t=L}^{L+S-1}\pi_{i t} z_{i t}$. For $\rho\in(0,1)$, the predicted $\rho$-quantile of $Z_{i; (L,S)}$ is denoted by $\hat{Z}^{\rho}_{i; (L,S)}$. These predictions are obtained from the sample paths by first summing over the span, then estimating the quantile by way of sorting. The {\em $\rho$-quantile loss}\footnote{
  $\Ex_Z[L_{\rho}(Z,\hat{z})]$ is minimized by the $\rho$-quantile. Also, $L_{0.5}(z,\hat{z}) = |z-\hat{z}|$.}
is defined as $L_{\rho}(z,\hat{z}) = 2(z-\hat{z})( \rho\Ind{z>\hat{z}} - (1-\rho)\Ind{z\le \hat{z}} )$. The {\em $\text{P}(\rho\cdot 100)$ risk metric} for $[L,L+S)$ is defined as $R^{\rho}[{\cal I}; (L, S)] = |{\cal I}|^{-1}\sum_{i\in{\cal I}} L_{\rho}(Z_{i; (L,S)}, \hat{Z}^{\rho}_{i; (L,S)})$, where the left argument $Z_{i; (L,S)}$ is computed from test targets.\footnote{
  More precisely, we filter ${\cal I}$ before use in $R^{\rho}[{\cal I}; (L, S)]$: ${\cal I}' = \{ i\in{\cal I}\, |\, \sum_{t=L}^{L+S-1}\pi_{i t}\ge 0.8 S \}$.}
We focus on P50 risk ($\rho=0.5$; mean absolute error) and P90 risk ($\rho=0.9$; the $0.9$-quantile is often relevant for automated ordering).

\begin{figure}[ht!]
\begin{minipage}[c]{0.4\textwidth}
    \centering
\footnotesize
\begin{tabular}{ccccc} 
\midrule
 & \parts & \articlesSubset & \articlesAll \\\midrule
$\#$ items & 19874 & 39700 & 534884\\
Unit $t$ & month & day & day \\
Median $CV^2$ & 2.4 & 5.8 & 9.7 \\
Freq.\ $z_t=0$ & 54\% & 46\% & 83\% \\
In-stock ratio & 100\% & 73\% & 71\% \\
Avg.\ size series & 33 & 329 & 293 \\
$\#$ item-days & 656K & 13M & 157M\\\midrule
\end{tabular}
\end{minipage}
\hfill
\begin{minipage}[c]{0.48\textwidth}
  \includegraphics[width=\textwidth]{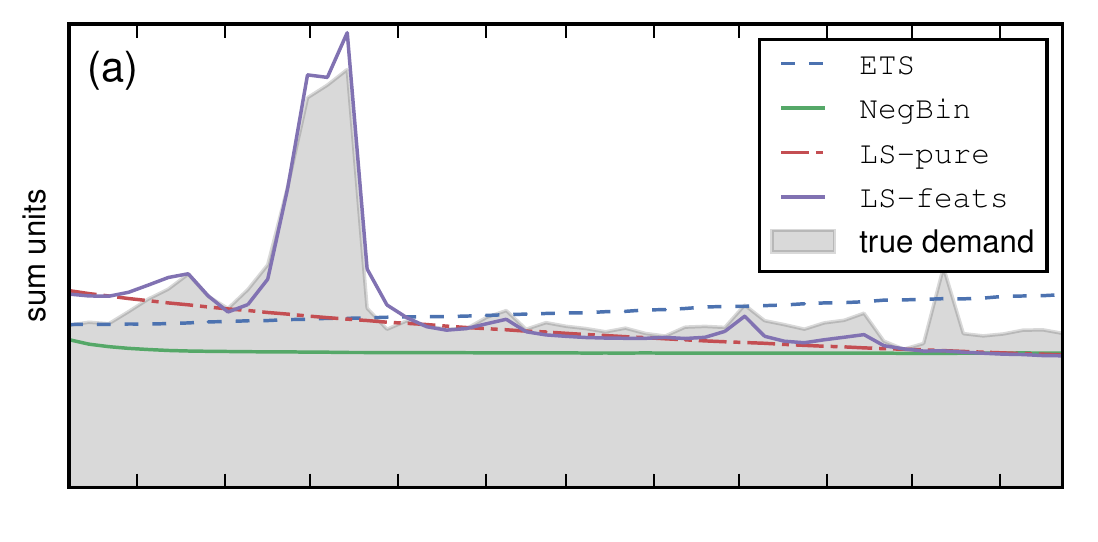}
\end{minipage}

\begin{minipage}[c]{0.5\textwidth}
  \includegraphics[width=\textwidth]{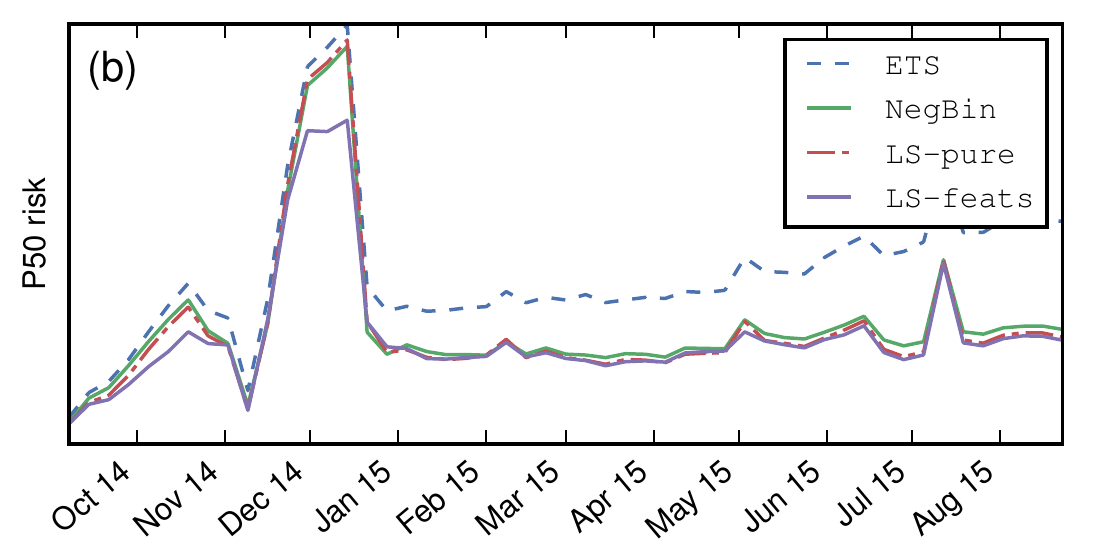}
\end{minipage}
\begin{minipage}[c]{0.50\textwidth}
  \includegraphics[width=\textwidth]{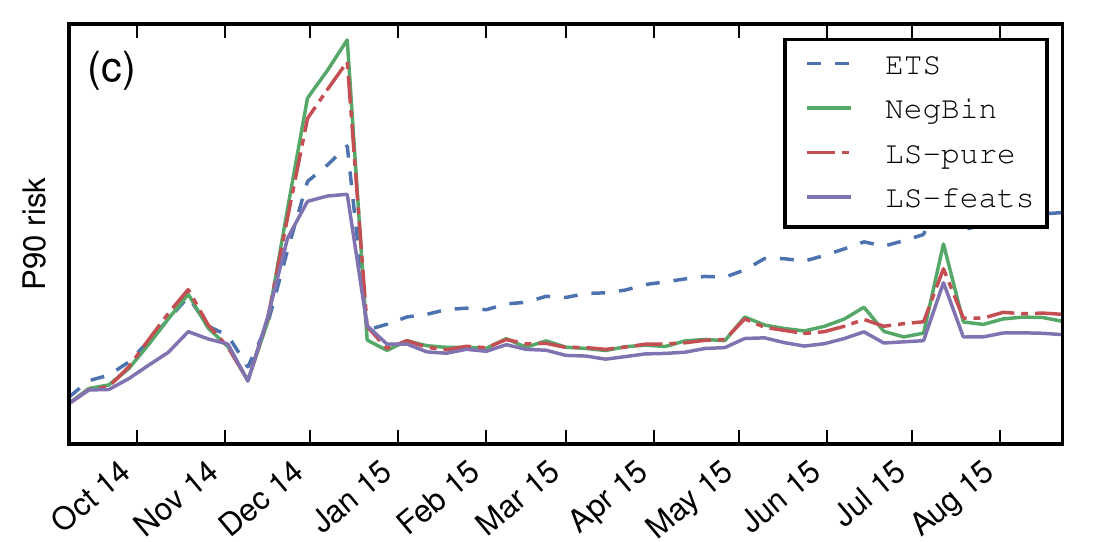}
\end{minipage}

\caption{\label{fig:stats-and-curves}
{\bf Table}: Dataset properties. $CV^2 = \Var[z_t] / \Ex[z_t]^2$ measures burstiness. 
{\bf (a)}: Sum of weekly P50 point (median) forecast over a one-year prediction range for the different methods (lines) as well as sum of true demand (shaded area), on dataset ${\cal I} = $ \articlesSubset{}.
{\bf (b)}: Weekly P50 risk $R^{0.5}[{\cal I}; (7\cdot k, 7)]$, $k=0,1,\dots$, for same dataset. {\bf (c)}: Same as (b) for P90 risk.
}
\end{figure}

We plot the P50 and P90 risk on dataset \articlesSubset{}, as well the sum of P50 point (median) forecast and the true demand, in the three panels of \figref{stats-and-curves}. All methods work well in the first week, but there are considerable differences further out. Naturally, losses are highest during the Christmas peak sales period. \issmfeatures{} strongly outperforms all others in this critical region (see \figref{stats-and-curves}, top right), by means of its features (holidays, seasonality). The Gaussian predictive distributions of \ets{} exhibit growing errors over time. With the exception of the Christmas period, \snyder{} works rather well (in particular in P50 risk), but is uniformly outperformed by both \issm{}, and \issmfeatures{} in particular.

A larger range of results are given in \tabref{parts-articlesSubset} (\parts{}, \articlesSubset{}) and \tabref{articlesAll} (\articlesAll{}), where numbers are relative to \snyder{}. Note that the R code for \ets{} could not be run on the large \articlesAll{}. On \parts{}, \snyder{} works best, yet \issm{} comes close (we did not use features on this dataset). On \articlesSubset{}, \issmfeatures{} outperforms all others in all scenarios. The featureless \snyder{} and \issm{} are comparable on this dataset. On the largest set \articlesAll{}, \issmfeatures{} generally outperforms the others, but differences are smaller.

\begin{table}[h]
\centering
\scriptsize
\setlength{\tabcolsep}{5pt}
\begin{tabular}{l|rrrr|rrrrrr}
\toprule
                & \multicolumn{4}{c|}{\parts}   & \multicolumn{6}{c}{\articlesSubset} \\
                & \multicolumn{2}{c}{P90 risk} & \multicolumn{2}{c|}{P50 risk}       & \multicolumn{3}{c}{P90 risk} & \multicolumn{3}{c}{P50 risk} \\
$(L,S)$         & $(0,2)$ & dy(8) & $(0,2)$ & dy(8) & $(0,56)$ & $(21,84)$ & wk(33) & $(0,56)$ & $(21,84)$ & wk(33) \\
\midrule
\ets{}          & 1.04 & 1.04 &  1.19  & 1.38                                & 0.99 & 0.75 & 1.13 & 1.07 & 1.10 & 1.18 \\
\issm{}         & 1.08 & 1.06 &  1.04  & 1.06                                & 1.07 & 0.97 & 0.99 & 0.95 & 1.03 & 0.99 \\
\issmfeatures{} &  --  &  --  &   --   &   --                                & {\bf 0.80} & {\bf 0.73} & {\bf 0.85} & {\bf 0.84} & {\bf 0.84} & {\bf 0.94} \\
\snyder{}       & {\bf 1.00} & {\bf 1.00} &  {\bf 1.00}  & {\bf 1.00}                                & 1.00 & 1.00 & 1.00 & 1.00 & 1.00 & 1.00 \\
\bottomrule
\end{tabular}
\vspace{2ex}

\caption{\label{tab:parts-articlesSubset}
  Results for dataset \parts{} (left) and \articlesSubset{} (right). Metric values relative to \snyder{} (each column). dy(8): Average of $R^{\rho}[{\cal I}; (k, 1)]$, $k=\rng[0]{7}$. wk(33): Average of $R^{\rho}[{\cal I}; (7\cdot k, 7)]$, $k=\rng[0]{32}$.}
\end{table}

\begin{table}[h]
\small
\centering
\begin{tabular}{l|rrrrrr}
\toprule
{} & \multicolumn{3}{c}{P90 risk} & \multicolumn{3}{c}{P50 risk} \\
$(L,S)$         & $(0,56)$ & $(21,84)$ & wk(33)& $(0,56)$ & $(21,84)$ & wk(33) \\
\midrule
\issm{}         & 1.11 & 1.03 & 0.99 & 1.00 & 1.03 & 1.05 \\
\issmfeatures{} & {\bf 0.95} & {\bf 0.86} & {\bf 0.89} &{\bf  0.92} & {\bf 0.88} & {\bf 0.98 }\\
\snyder{}       & 1.00 & 1.00 & 1.00 & 1.00 & 1.00 & 1.00 \\
\bottomrule
\end{tabular}
\vspace{2ex}

\caption{\label{tab:articlesAll}
  Results for dataset \articlesAll{}. Metric values relative to \snyder{} (each column). \ets{} could not be run at this scale.}
\end{table}

Finally, we report running times of parameter learning (outer optimization) for \issmfeatures{} on \articlesSubset{}. L-BFGS was run with {\tt maxIters = 55}, {\tt gradTol = $10^{-5}$}. Our experimental cluster consists of about 150 nodes, with {\tt Intel Xeon E5-2670} CPUs (4 cores) and 30GB RAM. Profiling was done separately in each stage: $k=0$ ($P5 = 0.180s$, $P50 = 1.30s$, $P95 = 2.15s$), $k=1$ ($P5 = 0.143s$, $P50 = 1.11s$, $P95 = 1.79s$), $k=2$ ($P5 = 0.138s$, $P50 = 1.29s$, $P95 = 3.25s$). Here, we quote median (P50), $5\%$ and $95\%$ percentiles (P5, P95). The largest time recorded was $10.4s$. The narrow spread of these numbers witnesses the robustness and predictability of the nested optimization process, crucial properties in the context of production systems running on parallel compute clusters.

\subsection{Modeling seasonality via latent state vs features}\label{sec:latent-vs-seasonality-exp}

So far we modeled seasonal effects via features.
In this experiment we show why modeling seasonality pattern using features alone is not sufficient.
For this, we selected one of the time series from the \verb|electricity| dataset \citep{Yu:16} which contains hourly time series of the electricity use (kW) of $370$ customers.
We use a single stage with Gaussian likelihood for this experiment.
Moreover, we use first $120$ hours for training and $120$ hours for testing.
On the left panel of Figure \ref{fig:latent-vs-features}, we show the result of modeling the hourly seasonality pattern using an indicator feature vector.
Note that forecast start time is marked by a vertical line in the plot.
The plot shows the median as well as P10 and P90 percentiles of the forecast distribution.
Note that the feature-based model only learns the average electricity usage per hour across the five cycles of the training data (see the electricity usage values in the first few hours of everyday in the prediction range).
On the right panel we show the result of modeling the hourly seasonality via latent state.
Note that the predictions here clearly reflect the most recent behavior of the time series.

\begin{figure*}[th]
\begin{center}
\includegraphics[scale=0.28]{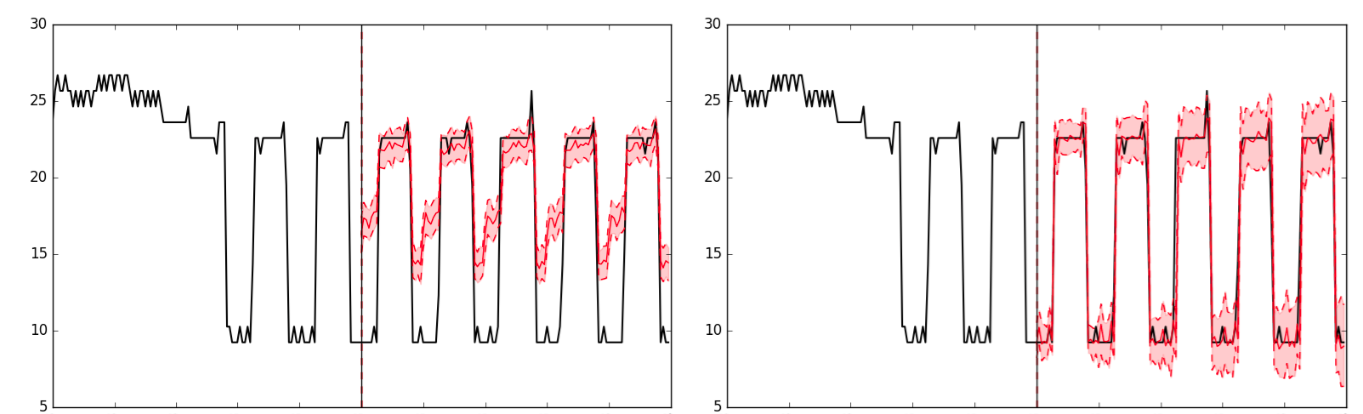}
\caption{Modeling hourly seasonality via features (left panel) versus latent state (right panel).}
\label{fig:latent-vs-features}
\end{center}
\end{figure*}

\subsection{Modeling user-defined seasonality pattern}\label{sec:custom-seasonality}

\begin{figure*}[th]
\begin{center}
\includegraphics[scale=0.28]{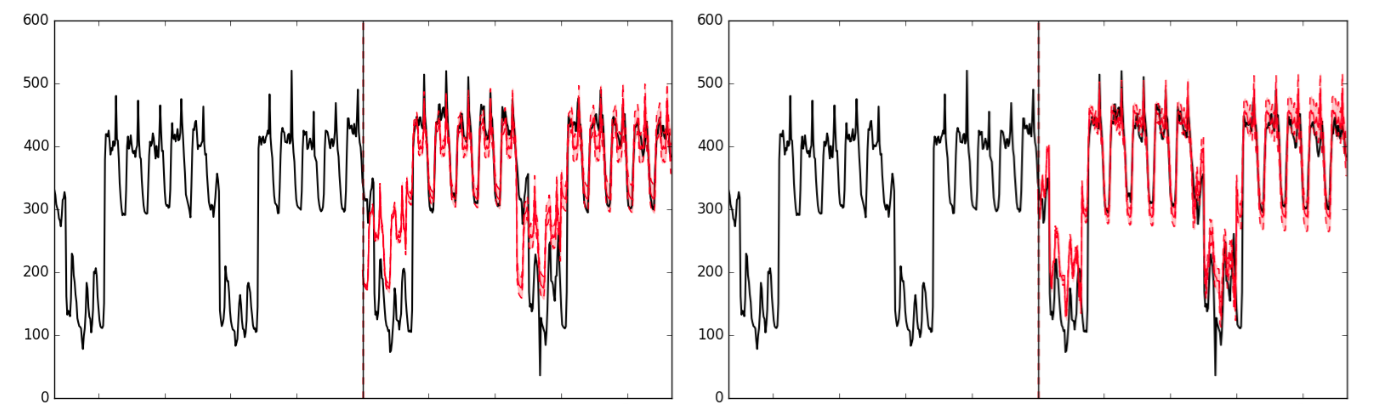}
\caption{Non-standard pattern: modeling by a combination of hourly and daily seasonal models (left panel); modeling by user-defined pattern (right panel).}
\label{fig:custom}
\end{center}
\end{figure*}

One attractive feature of the proposed latent state model is that it allows users to encode custom periodic seasonality inherent in the data.
We will illustrate this using one of the time series from the \verb|electricity| dataset shown in Figure \ref{fig:custom}.
Again we use a single stage with Gaussian likelihood for this experiment.
The model is trained on the first two weeks of hourly data and tested on the remaining two weeks.
Notice that the target clearly has an hourly seasonality pattern (looking at the training part of the time series).
However there is an exception to this pattern on two days of every week (Saturdays and Sundays).
Combining hourly seasonality ISSM with that of daily seasonality results in the forecasts shown in the left panel of Figure \ref{fig:custom}.
The forecasts on Saturdays and Sundays are way off since the hourly pattern on these days is not additive to the hourly pattern on the other days and hence cannot be captured by simply adding hourly and daily seasonality models.
To properly model this pattern, we define the factors as the cross product of hour and day; i.e., there is a factor for every hour and day pair. So, the length of the seasonality and the dimension of the latent state is given by $24 \times 7 = 168$.
We further combine Monday to Friday into a single group since the hourly pattern is similar on these days reducing the latent state dimension to $72$.
The corresponding result is shown in the right panel of Figure \ref{fig:custom} which shows that the model correctly picked up the seasonality for all days.

\section{Conclusions. Future Work}\label{sec:conclus}

In this paper, we developed a framework for maximum likelihood learning of probabilistic latent state forecasting models, which can be seen as principled time series extensions of generalized linear models. We pay special attention to the intermittent and bursty statistics of demand, characteristic for the vast inventories maintained by large retailers or e-commerce platforms. We show how approximate Bayesian inference techniques can be implemented in a robust and highly scalable way, so to enable a forecasting system which runs safely on hundred of thousands of items and hundreds of millions of item-days.

We can draw some conclusions from our comparative study on a range of real-world datasets. Our proposed method strongly outperforms competitors on sales data from fast and medium moving items. Besides good short term forecasts due to temporal smoothness and well-calibrated growth of uncertainty, our use of a feature vector seems most decisive for medium term forecasts. On slow moving items, simpler methods like \snyder{} \citep{Snyder:12} are competitive, even though they lack signal models which could be learned from data.

We are investigating several directions for future work. Our current system uses time-independent ISSMs, in particular $\vg{t}=[\alpha]$ means that the same amount of innovation variance is applied every day. This assumption is violated by our data, where a lot more variation happens in the weeks leading up to Christmas or before major holidays than during the rest of the year. To this end, we are exploring learning two parameters: $\alpha_h$ during high-variation periods, and $\alpha_l$ for all remaining days. 

One of the most important future directions is to learn about and exploit dependencies between the demand time series of different items. In fact, the strategy to learn and forecast each item independently is not suitable for items with a short demand history, or for slow moving items. One approach we pursue is to couple latent processes by a shared (global) linear or non-linear function.

\appendix

\section*{Appendix: Overview}
We collect all technical details as well as extensions of the proposed method here.
The appendix is organized as follows. 
First, in Appendix~\ref{app:Inference}, we discuss the posterior inference problem in the Gaussianized latent state model. 
Here we derive the forward (filtering equations) and backward messages (smoothing equations) which are posterior marginals of the latent variables conditioned on observations.
Appendix~\ref{app:ISF} then discusses the numerically stable implementation of these messages using the square root information filter \citep{Hyndman:08} along with some enhancements.
In a pure Bayesian inference one would integrate over the feature weights $\vw{}$ instead of estimating them as part of the parameters.
Appendix~\ref{app:InfOverW} discusses this inference over the weights $\vw{}$.
Our model can easily be extended to the case of missing observations which is discussed in Appendix~\ref{app:MissingObs}.

Recall that we need to find the mode of the complete likelihood (Algorithm \ref{alg:Mode}) for the Laplace approximation. 
This is done using Newton's method where the Newton's step is reduced Kalman smoothing.
Since the full Newton step can overshoot we discuss the line search as well as good initialization for this minimization problem in Appendix~\ref{app:LineSearch}.
Finally, in Appendix~\ref{app:Gradient} we discuss the computation of the gradient w.r.t the parameters $\vth{}$ of the outer criterion $\psi(\vth{})$, which is minimized by L-BFGS.

\section{Inference in the Latent State Model}
\label{app:Inference}

Each iteration of mode finding algorithm (Algorithm \ref{alg:Mode}) requires the minimizer of $\tilde{F}(\veps{},\vl{0})$ which is the negative log-joint of a linear-Gaussian model with latent variables $\veps{}$, $\vl{0}$ and observations $\tvz{}$.
Since mean and mode are identical for such models, this step is equivalent to computing the posterior means $\Ex[\veps{} | \tvz{}]$ and $\Ex[\vl{0} | \tvz{}]$  in this model.
We now focus on this posterior inference in the linear-Gaussian model for fixed potentials $N(\tsz{t} | y_t,\sigma_t^2)$.
In order to make the posterior inference in the ``Gaussianized'' model tractable, we need to represent $\vl{t}, t = 1, \ldots, T - 1,$ also as latent variables. 
Let ${\cal D} = \{\tsz{1}, \ldots, \tsz{T}\}$ and $\tsz{\le t} = \{\tsz{1}, \ldots, \tsz{t}\}$.
Our goal is to compute $P(\vl{0} | {\cal D})$.
The graphical representation of the Gaussianized model is given in Figure \ref{fig:graphical-model}.
The moral graph for this model has cliques $\{\vl{t-1}, \vl{t}, \eps_t \}$, and messages are joint distributions over these. However, at least the forward messages, which are posterior distributions conditioned on $\tsz{\le t }$, have the following restricted form:
\begin{align}\label{eq:forward_messages}
q_t(\vl{t-1},\eps_t,\vl{t}) = P(\vl{t} | \vl{t-1},\eps_t) N(\eps_t |
0,1) q_t(\vl{t-1}),
\end{align}
where $q_t(\vl{t-1})$ is a multivariate Gaussian. This is the only distribution we have to propagate forward. 
We first derive the forward messages $q_{t + 1}(\vl{t}), t = 0, \ldots, T - 1$. 
Note that these are the standard Kalman filtering equations for the innovation state space model.

  \begin{figure}
  \centering
    \includegraphics[scale=0.42]{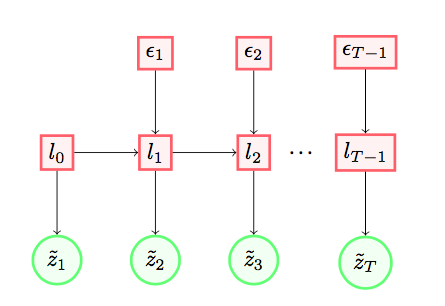}  
    \caption{Graphical representation of the Gaussianized latent state model} 
    \label{fig:graphical-model}    
  \end{figure}  

\begin{proposition}
	Let $q_{t+1}(\vl{t})$ denote the marginal distribution $P(\vl{t} | \tsz{\le t +1})$ for $t = 0, \ldots, T-1$, where $\tsz{\le t} = \{\tsz{1}, \ldots, \tsz{t}\}$. 
	Then $q_{t + 1}(\vl{t}) = N(\vl{t} | \vmu{t}, \mxsigma{t}), \forall t$.
	The mean $\vmu{t}$ and the covariance $\mxsigma{t}$ for $t = 1, \ldots, T - 1$ are given by the forward recursion
	\begin{align*}
	\vmu{t} 	&= \mxf{} \vmu{t - 1} + \vk{t} (\tsz{t + 1} - b_{t + 1} - \va{t+ 1}^T \mxf{} \vmu{t - 1}) \\
	\mxsigma{t} &= (\mxi{} - \vk{t} \va{t + 1}^T) \mxp{t - 1}
	\end{align*}
	where we used the notation
	\begin{align*}
	\mxp{t - 1} &= \mxf{} \mxsigma{t - 1} \mxf{}^T + \vg{t} \vg{t}^T\\
	\vk{t} &= (\va{t + 1}^T \mxp{t - 1} \va{t + 1} + \sigma^2_{t+1})^{-1} \mxp{t - 1} \va{t + 1} 
	\end{align*}
	The initial conditions ($t = 0$) for the recursion are given by
	\begin{align*}
	\vmu{0} &= \bvmu{0} + \vk{0} (\tsz{1} - b_1 - \va{1}^T \vmu{0})\\
	\mxsigma{0} &= (\mxi{} - \vk{0} \va{1}^T) \bmxsigma{0},
	\end{align*}
	where $\bvmu{0}, \bmxsigma{0}$ are respectively the mean and covariance of the prior $P(\vl{0})$ on the initial state and $\vk{0} = \bmxsigma{0} \va{1} (\sigma_1^2  + \va{1}^T \bmxsigma{0} \va{1})^{-1}$.
\end{proposition}
\begin{proof}
	In the proof, we write $N(\tsz{t} | \vl{t-1})$ for the conditional Gaussian $N(\tsz{t} | y_t,\sigma_t^2)$, where
	\[
	y_{t} = \va{t}^T\vl{t-1} + b_t,\quad b_t = \vx{t}^T\vw{}.
	\]
	The forward pass computes the marginals $q_t(\vl{t - 1})$ for $t = 1, \ldots, T$.
	First, 
	\[
	q_{1}(\vl{0})  = P(\vl{0} | \tsz{1})
	= \frac{P(\vl{0}) N(\tsz{1} | \vl{0})} {P(\tsz{1})}.
	\]
	Since we assume $P(\vl{0})$ is Gaussian given by $N(\vl{0} | \bvmu{0}, \bmxsigma{0})$, it follows that $q_1(\vl{0})$ is also Gaussian.
	Its covariance $\mxsigma{0}$ and mean $\vmu{0}$ can be derived, using standard techniques for manipulating Gaussians, as
	\begin{align*}
	\mxsigma{0} &= \left(\bmxsigma{0}^{-1} + \va{1} (\sigma_1^2)^{-1} \va{1}^T\right)^{-1}\\
	&\stackrel{(s)}{=} \bmxsigma{0} - \bmxsigma{0} \va{1} (\sigma_1^2  + \va{1}^T \bmxsigma{0} \va{1})^{-1} \va{1}^T \bmxsigma{0}\\
	&= (\mxi{} - \vk{0} \va{1}^T) \bmxsigma{0}\\		
	\vmu{0} 	&= \mxsigma{0} \left(\va{1} (\sigma_0^2)^{-1} (\tsz{1} - b_1) + \bmxsigma{0}^{-1} \bvmu{0}  \right)\\
	&= (\mxi{} - \vk{0} \va{1}^T)
	\left(\bmxsigma{0} \va{1} (\sigma_1^2)^{-1} (\tsz{1} - b_1) + \bvmu{0}  \right)\\
	&= (\mxi{} - \vk{0} \va{1}^T)
	\bmxsigma{0} \va{1} (\sigma_1^2)^{-1} (\tsz{1} - b_1) + \bvmu{0} - \vk{0} \va{1}^T \bvmu{0}\\
	&= \bmxsigma{0} \va{1}( 1 - (\sigma_1^2  + \va{1}^T \bmxsigma{0} \va{1})^{-1} \va{1}^T \bmxsigma{0} \va{1} )
	(\sigma_1^2)^{-1} (\tsz{1} - b_1) + \bvmu{0} - \vk{0} \va{1}^T \bvmu{0}\\
	&= \vk{0} (\tsz{1} - b_1) + \bvmu{0} - \vk{0} \va{1}^T \bvmu{0}\\
	&= \bvmu{0} + \vk{0} (\tsz{1} - b_1 - \va{1}^T \vmu{0}),
	\end{align*}
	where in step (s), we used the matrix identity $(\mxa{} + \mxb{} \mxd{}^{-1} \mxc{})^{-1} = \mxa{}^{-1} - \mxa{}^{-1} \mxb{} (\mxd{} + \mxc{} \mxa{}^{-1} \mxb{})^{-1} \mxc{} \mxa{}^{-1}$.
	
	For forward recursion, we have for $t = 1, \dots, T - 1$, 
	\begin{align*}
	q_{t+1}(\vl{t})  = P(\vl{t} | \tsz{\le t + 1})
	= \frac{P(\vl{t}, \tsz{t + 1} | \tsz{\le t})} {P(\tsz{t + 1} | \tsz{\le t})}.
	\end{align*}	
	By rearranging the terms, we get
	\begin{align*}
	P(\tsz{t + 1} | \tsz{\le t}) q_{t+1}(\vl{t}) &= 
	P(\tsz{t + 1} | \vl{t}, \tsz{\le t}) P(\vl{t} | \tsz{\le t}) \\
	&= N(\tsz{t+1} | \vl{t}) \int q(\vl{t}, \vl{t - 1}, \eps_t)\, d\eps_t d\vl{t-1}\\
	&= N(\tsz{t+1} | \vl{t}) \int P(\vl{t} | \vl{t-1},\eps_t) N(\eps_t | 0,1) q_t(\vl{t-1})\, d\eps_t d\vl{t-1},
	\end{align*}
	where in the last step we used (\ref{eq:forward_messages}).
	Note that $P(\vl{t} | \vl{t - 1}, \eps_t)$ is a delta distribution given by 
	$\vl{t} = \mxf{} \vl{t - 1} + \vg{t} \eps_t$.
	The integral is then given by
	\[ \vl{t} = \mxf{} \vl{t - 1} + \vg{t} \eps_t,\ \vl{t - 1} \sim q_t,\ \eps_t \sim N(0, 1). \]
	Hence, the above equality reduces to 
	\begin{align*} 
	P(\tsz{t + 1} | \tsz{\le t}) q_{t+1}(\vl{t}) &= N(\tsz{t+1} | \va{t + 1}^T\vl{t} + b_{t+1}, \sigma_{t + 1}^2)\ N(\vl{t} | \mxf{} \vmu{t - 1}, \mxp{t-1}),
	\end{align*}	
	where we used the notation $\mxp{t - 1} = \mxf{} \mxsigma{t - 1} \mxf{}^T + \vg{t}\vg{t}^T$.
	This equation represents the reformulation of the joint distribution $P(\tsz{t + 1}, \vl{t} | \tsz{\le t})$ as the equivalent factors $P(\tsz{t + 1} | \tsz{\le t}) P(\vl{t}| \tsz{\le t + 1})$ and $P(\vl{t} | \tsz{\le t}) P(\tsz{t + 1} | \vl{t}, \tsz{\le t})$.
	Since the factors on the right hand side are Gaussians, it follows that $P(\tsz{t + 1} | \tsz{\le t})$ and $q_{t+1}(\vl{t})$ are Gaussians as well.
	Using the Bayesian theorem for Gaussians, we obtain
	\[ P(\tsz{t + 1} | \tsz{\le t}) = N(\tsz{t+1} | \va{t + 1}^T \mxf{} \vmu{t - 1} + b_{t+1}, \va{t + 1}^T \mxp{t - 1} \va{t + 1} + \sigma_{t + 1}^2). \]
	The covariance $\mxsigma{t}$ of $q_{t + 1}(\vl{t})$ is given by
	\begin{align*} 
	\mxsigma{t} &= (\mxp{t - 1}^{-1} + (\sigma^2_{t + 1})^{-1} \va{t + 1} \va{t + 1}^T)^{-1}\\
	&= \mxp{t - 1} - \mxp{t - 1} \va{t + 1} (\sigma^2_{t + 1} + \va{t + 1}^T \mxp{t - 1} \va{t + 1}) ^{-1} \va{t + 1}^T \mxp{t - 1}\\
	&= (\mxi{} - \vk{t} \va{t + 1}^T) \mxp{t - 1}
	\end{align*}
	where for the second equality, we used the matrix identity $(\mxa{} + \mxb{} \mxd{}^{-1} \mxc{})^{-1} = \mxa{}^{-1} - \mxa{}^{-1} \mxb{} (\mxd{} + \mxc{} \mxa{}^{-1} \mxb{})^{-1} \mxc{} \mxa{}^{-1}$ and for the last equality we used the notation $\vk{t} = (\sigma^2_{t+1} + \va{t + 1}^T \mxp{t - 1} \va{t + 1})^{-1} \mxp{t - 1} \va{t + 1}$.
	The mean $\vmu{t}$ of $q_{t + 1}(\vl{t})$ can be deduced as
	\begin{align*}
	\vmu{t} &= \mxsigma{t} \left(\va{t + 1} (\sigma^2_{t+1})^{-1} (\tsz{t + 1} - b_{t + 1}) + \mxp{t - 1}^{-1} \mxf{} \vmu{t - 1} \right) \\
	&= (\mxi{} - \vk{t} \va{t+ 1}^T) \left(\va{t + 1} (\sigma^2_{t+1})^{-1} (\tsz{t + 1} - b_{t + 1}) \mxp{t - 1} + \mxf{} \vmu{t - 1} \right)\\
	&= \vk{t} (\tsz{t + 1} - b_{t + 1}) + \mxf{} \vmu{t - 1} - \vk{t} \va{t+ 1}^T \mxf{} \vmu{t - 1}\\
	&= \mxf{} \vmu{t - 1} + \vk{t} (\tsz{t + 1} - b_{t + 1} - \va{t+ 1}^T \mxf{} \vmu{t - 1}).
	\end{align*}
\end{proof}
Next we derive the backward messages $r_t(\vs{t})$, which are the  posterior marginals $P(\vs{t} | {\cal D})$ for the latent variables $\vs{t} = [\vl{t}, \epsilon_t]$, $t = 1, \ldots, T-1$.
\begin{proposition}	
	Let $r_t(\vs{t})$ denote the posterior marginals $P(\vs{t} | {\cal D})$ of the latent variables $\vs{t} = [\vl{t}, \epsilon_t]$, $t = 1, \ldots, T-1$.
	Then we have
	\[ r_t(\vs{t}) = q_t(\eps_t | \vl{t}) r_t(\vl{t}),\]
	where
	$q_t(\epsilon_t | \vl{t}) = N(\eps_t | m_{t}^{\eps} + \vd{t}^T \vl{t}, v_{t}^\eps)$, 
    $\vd{t} =  v_{t}^{\eps} \mxf{}^{-T} \mxsigma{t-1}^{-1} \vh{t} $,
	$m_{t}^\eps = - v_{t}^\eps \vmu{t - 1}^T \mxsigma{t - 1}^{-1} \vh{t}$ 
	and $v_{t}^{\epsilon} = (1 + \vh{t}^T \mxsigma{t-1}^{-1} \vh{t})^{-1}$, $\vh{t} = \mxf{}^{-1} \vg{t}$.
	Moreover, $r_t(\vl{t}) = N(\vl{t} | \hvmu{t}, \hmxsigma{t}), \forall t$. 
	The mean $\hvmu{t}$ and the covariance $\hmxsigma{t}$ for $t = 0, \ldots, T-2$ are given by the backward recursion
	\begin{align*}
	\hvmu{t} 		&= \mxf{}^{-1} \hvmu{t+1} - \vh{t+1} \vd{t+1}^T \hvmu{t+1} -  m_{t+1}^\eps \vh{t+1}\\
	\hmxsigma{t} 	&= \mxf{}^{-1}(\Id - \vg{t+1} \vd{t+1}^T )  \hmxsigma{t+1}  (\Id - \vg{t+1} \vd{t+1}^T )^T \mxf{}^{-T} + \vh{t+1} v_{t+1}^\eps  \vh{t+1}^T.
	\end{align*}
	The initial condition ($t = T - 1$) for the recursion is given by
	\[  r_{T-1}(\vl{T-1}) = q_T(\vl{T-1}). \]
	The posterior over the initial state $P(\vl{0}|\cal D)$ is given by $r_0(\vl{0})$.
\end{proposition}
\begin{proof}
	We have 
	\begin{align*}
	P(\vl{t}, \eps_t | {\cal D}) 	= P(\vl{t} | {\cal D}) P(\eps_t | \vl{t}, {\cal D}) 
	= r_t(\vl{t}) q_t(\eps_t | \vl{t})
	\end{align*}
	where we used the conditional independence property
	\[ P(\eps_t | \vl{t}, {\cal D}) =  P(\eps_t | \vl{t}, \tsz{\le t}) \]
	
	Now we derive the distribution for $q_t(\eps_t | \vl{t})$.
	Note that if $q_t(\vl{t-1}) = N(\vmu{t-1},\mxsigma{t-1})$, then $q_t(\vl{t} | \eps_t)$ is a conditional Gaussian with
	\[
	\Ex[\vl{t}|\eps_t] = \mxf{}\vmu{t-1} + \vg{t}\eps_t,\quad \Cov[\vl{t}|\eps_t]
	= \mxf{}\mxsigma{t-1}\mxf{}^T.
	\]
	Then, we have
	\begin{align*}
	q_t(\vl{t}, \eps_t)		&= N(\vl{t} | \mxf{} \vmu{t - 1} + \vg{t} \eps_t, \mxf{} \mxsigma{t - 1} \mxf{}^T )  N(\eps_t | 0, 1)\\
	&\propto \exp\left[ -\frac{1}2\eps_t^2 \left( 1 +
	\vh{t}^T\mxsigma{t-1}^{-1}\vh{t} \right) + \left( \mxf{}^{-1}\vl{t} - \vmu{t-1}
	\right)^T \mxsigma{t-1}^{-1}\vh{t} \eps_t \right],
	\end{align*}
	where we used notation $\vh{t} = \mxf{}^{-1}\vg{t}$.
	From this joint distribution, we see that the conditional distribution $q_t(\eps_t | \vl{t})$ is given by
	\begin{align*}
 	 q_t(\epsilon_t | \vl{t}) &= N(\eps_t | v_{t}^{\epsilon} \left( \mxf{}^{-1}\vl{t} - \vmu{t-1}
				\right)^T \mxsigma{t-1}^{-1}\vh{t}, v_{t}^{\epsilon})\\
     &= N(\eps_{t} | v_{t}^{\eps} \vl{t}^T \mxf{}^{-T} \mxsigma{t-1}^{-1} \vh{t} - v_{t}^\eps \vmu{t-1}^T \mxsigma{t-1}^{-1} \vh{t}, v_{t}^\eps),
	\end{align*}
	where $v_{t}^{\epsilon} = (1 + \vh{t}^T \mxsigma{t-1}^{-1} \vh{t})^{-1}$.
	Let $\vd{t} =  v_{t}^{\eps} \mxf{}^{-T} \mxsigma{t-1}^{-1} \vh{t} $.
	Then $q_{t}(\epsilon_{t} | \vl{t})$ is given by
	\begin{align}\label{eq:error}
	 \epsilon_{t} = \vd{t }^T \vl{t} + \nu_{t}^{\eps},\quad \nu_{t}^{\eps}\sim N(m_{t}^{\eps},
	v_{t}^{\eps}),
	\end{align}
	where $m_{t}^\eps = - v_{t}^\eps \vmu{t-1}^T \mxsigma{t-1}^{-1} \vh{t}$.
	
	Note that $r_{T - 1}(\vl{T - 1}) = q_{T}(\vl{T - 1})$. 
	Next, for the backward recursion, we have for $t = 0, \ldots, T - 2$,
	\begin{align*}
	r_t(\vl{t}) &= \int r_t(\vl{t + 1}, \vl{t}, \eps_{t + 1})\ d\eps_{t+1} d\vl{t+1}, \\
	&= \int r_t(\vl{t + 1}) r_t(\eps_{t + 1} | \vl{t + 1}) r_t(\vl{t} | \vl{t + 1}, \eps_{t + 1}) \ d\eps_{t+1} d\vl{t+1}, \\
	&= \int r_{t}(\vl{t+1}) q_{t+1}(\eps_{t+1} | \vl{t+1}) q_{t+1}(\vl{t} | \vl{t+1}, \eps_{t+1})  \ d\eps_{t+1} d\vl{t+1},
	\end{align*}
	where we used the conditional independence properties 
	\begin{align*}
	r_t(\eps_{t + 1} | \vl{t + 1}) 			&= q_{t+1}(\eps_{t+1} | \vl{t+1})\\
	r_t(\vl{t} | \vl{t + 1}, \eps_{t + 1}) 	&= q_{t+1}(\vl{t} | \vl{t+1}, \eps_{t+1}) 
	\end{align*}
	Note that $q_{t+1}(\vl{t} | \vl{t+1}, \eps_{t+1})$ is a delta distribution given by $\vl{t} = \mxf{}^{-1}( \vl{t+1} - \vg{t+1}\eps_{t+1} )$.
	Then $r_t(\vl{t})$ is given by 
	\begin{align*}
     \vl{t} &= \mxf{}^{-1} \vl{t+1} - \mxf{}^{-1} \vg{t+1} \epsilon_{t+1},\quad \vl{t + 1} \sim r_{t + 1}(\vl{t + 1}), \eps_{t + 1} \sim q_{t + 1}(\eps_{t + 1} | \vl{t + 1})\\
        &\stackrel{(\ref{eq:error})}{=} \mxf{}^{-1} \vl{t+1} - \mxf{}^{-1}\vg{t+1} \vd{t+1}^T \vl{t+1} - \mxf{}^{-1} \vg{t+1}\nu_{t+1}^\eps, \quad \nu_{t+1}^{\eps}\sim N(m_{t+1}^{\eps}, v_{t+1}^{\eps})\\
        &= \mxf{}^{-1} (\Id - \vg{t+1} \vd{t+1}^T ) \vl{t+1} - \vh{t+1}\nu_{t+1}^\eps.
   \end{align*}	
%
	Thus, $r_t(\vl{t})$ is a Gaussian distribution whose mean and covariance are given by 
	\begin{align*} 
 	 \hvmu{t} &= \mxf{}^{-1} (\Id - \vg{t+1} \vd{t+1}^T ) \hvmu{t+1} -  m_{t+1}^\eps \vh{t+1} \\
 	 &= \mxf{}^{-1} \hvmu{t+1} - \vh{t+1} \vd{t+1}^T \hvmu{t+1} -  m_{t+1}^\eps \vh{t+1},\\
	\hmxsigma {t} &= \mxf{}^{-1}(\Id - \vg{t+1} \vd{t+1}^T )  \hmxsigma{t+1}  (\Id - \vg{t+1} \vd{t+1}^T )^T \mxf{}^{-T} + \vh{t+1} v_{t+1}^\eps  \vh{t+1}^T 
	\end{align*}
\end{proof}

\section{Square Root Information Filter}
\label{app:ISF}

For numerical stability, the forward messages $q_t(\vl{t-1})$ and the RTS backward messages $r_t(\vl{t})$ are computed using {\em square root information filter} \cite[ch.~12]{Hyndman:08}.
The rough idea is to represent a joint Gaussian distribution over $\vl{}$ as
\[
\mxr{}\vl{} = \vc{},\quad \vc{}\sim N(\vm{},(\diag\vv{})),
\]
where $\mxr{}$ is unit upper triangular (diagonal elements are 1). In our case, we would represent $q_t(\vl{t-1})$ by $\{ \mxr{t-1}, \vm{t-1}, \vv{t-1} \}$. This representation can be seen as a stochastic linear equation system, with deterministic matrix and Gaussian random vectors, where the right hand side consists of {\em independent} random variables.

The forward propagation works as follows. Let $\mxq{t}$ be orthogonal such that $\mxq{t}\vg{t} = \|\vg{t}\|\vdelta{1}$. $\mxq{t}$ can be obtained as product of Givens rotations, or a single Householder reflection. We use the equations
\[
\mxr{t-1}\vl{t-1} = \vc{t-1},\quad \vl{t|t-1} = \mxf{}\vl{t-1} + \eps_t\vg{t},
\quad \tsz{t+1|t} = \va{t+1}^T\vl{t|t-1} + \nu_{t+1},
\]
where $\nu_{t+1}\sim N(b_{t+1}, \sigma_{t+1}^2)$. This leads to the stochastic system
\begin{equation}\label{eq:stochsys-1}
\left[ \begin{array}{ccc}
\mxr{t-1} & \mxzero & \vzero \\
-\mxq{t}\mxf{} & \mxq{t} & \vzero \\
\vzero^T & -\va{t+1}^T & 1
\end{array} \right] \left[ \begin{array}{c}
\vl{t-1} \\
\vl{t|t-1} \\
\tsz{t+1|t}
\end{array} \right] = \left[ \begin{array}{c}
\vc{t-1} \\
\eps_t\|\vg{t}\|\vdelta{1} \\
\nu_{t+1}
\end{array} \right].
\end{equation}
Note that the right hand sides of the trailing equations have zero variances. We now use a procedure to be detailed below to transform this system into
\begin{equation}\label{eq:stochsys-2}
\left[ \begin{array}{ccc}
* & * & * \\
\vzero & \mxr{t} & \vr{} \\
\vzero^T & \vzero^T & 1
\end{array} \right] \left[ \begin{array}{c}
\vl{t-1} \\
\vl{t|t-1} \\
\tsz{t+1|t}
\end{array} \right] = \left[ \begin{array}{c}
* \\
\vc{t|t-1} \\
d_{t+1|t}
\end{array} \right].
\end{equation}
Here, $\mxr{t}$ is unit upper triangular, and the right hand sides are independent Gaussians. The entries marked as ``*'' are not important in the sequel and need not be computed. Now, the $d_{t+1|t}$ expression can be used in order to evaluate the (approximate) log likelihood. Moreover, $q_{t+1}(\vl{t})$ is obtained by plugging in $\tsz{t+1}$ for $\tsz{t+1|t}$:
\[
\vc{t} = \vc{t|t-1} - \tsz{t+1}\vr{}.
\]
This completes the description of the forward pass.

Alongside the forward pass, we need to determine $q_t(\eps_t | \vl{t})$, which is given by
\begin{equation}\label{eq:eps-given-l}
\eps_t = \vk{t}^T\vl{t} + \nu_t^{\eps},\quad \nu_t^{\eps}\sim N(m_t^{\eps},
v_t^{\eps}).
\end{equation}
Given that $q_t(\vl{t-1})$ is represented by
\[
\mxr{t-1}\vl{t-1} = \vc{t-1}\sim N(\vm{t-1},\mxv{t-1}),\quad \mxv{t-1} =
\diag\vv{t-1},
\]
we have that
\[
\mxsigma{t-1}^{-1} = \mxr{t-1}^T\mxv{t-1}^{-1}\mxr{t-1},\quad
\mxsigma{t-1}^{-1}\vmu{t-1} = \mxr{t-1}^T\mxv{t-1}^{-1}\vm{t-1}.
\]
Plugging this into the equations above:
\[
v_t^{\eps} = \left( 1 + \tvh{t}{}^T\mxv{t-1}^{-1}\tvh{t}
\right)^{-1},\quad \tvh{t} = \mxr{t-1}\mxf{}^{-1}\vg{t},
\]
and
\[
\vk{t} = v_t^{\eps} \mxf{}^{-T}\mxr{t-1}^T\mxv{t-1}^{-1}\tvh{t},\quad m_t^{\eps} =
-v_t^{\eps} \vm{t-1}^T\mxv{t-1}^{-1}\tvh{t}.
\]
Now the backward pass. Suppose that $r_{t}(\vl{t})$ is represented by $\mxr{t}\vl{t} = \vc{t}$. The relevant equation is
\[
\mxf{}\vl{t-1} = \vl{t} - \vg{t}\eps_{t} = \left( \Id - \vg{t}\vk{t}^T
\right) \vl{t} - \vg{t}\nu_{t}^{\eps}.
\]
Multiplying with $\mxq{t}$, we obtain the stochastic system
\[
\left[ \begin{array}{cc}
\mxr{t} & \mxzero \\
\mxq{t}(\Id - \vg{t}\vk{t}^T) & -\mxq{t}\mxf{}
\end{array} \right] \left[ \begin{array}{c}
\vl{t} \\
\vl{t-1}
\end{array} \right] = \left[ \begin{array}{c}
\vc{t} \\
\nu_{t}^{\eps}\|\vg{t}\|\vdelta{1}
\end{array} \right].
\]
The same procedure as above renders $\mxr{t-1}\vl{t-1} = \vc{t-1}$ as representation of $r_{t-1}(\vl{t-1})$. At this point, we use \eqp{eps-given-l} to obtain
\[
\Ex[\eps_t | \mathcal{D}] = \vk{t}^T\Ex[\vl{t} | \mathcal{D}] + m_t^{\eps} =
\vk{t}^T\mxr{t}^{-1}\vm{t} + m_t^{\eps}.
\]
Moreover, $\Ex[\vl{0} | \mathcal{D}] = \mxr{0}^{-1}\vm{0}$. Finally, we also require
\[
\Ex[y_t | \mathcal{D}] = \tva{t}{}^T\vm{t-1} + b_t,\quad \Var[y_t |
\mathcal{D}] = \tva{t}{}^T(\diag\vv{t-1})\tva{t},\quad \tva{t} = \mxr{t-1}^{-T}
\va{t}.
\]

Finally, the computation of $q_1(\vl{0})$. We assume that $P(\vl{0})$ is Gaussian with a diagonal covariance, given by $\vl{0} = \tvc{0}$, and $\nu_1\sim N(b_1, \sigma_1^2)$. The relevant system is
\[
\left[ \begin{array}{cc}
\Id & \mxzero \\
-\va{1}^T & 1
\end{array} \right] \left[ \begin{array}{c}
\vl{0} \\
\tsz{1|0}
\end{array} \right] = \left[ \begin{array}{c}
\tvc{0} \\
\nu_{1}
\end{array} \right].
\]
This is transformed to
\[
\left[ \begin{array}{cc}
\mxr{0} & \vr{} \\
\vzero^T & 1
\end{array} \right] \left[ \begin{array}{c}
\vl{0} \\
\tsz{1|0}
\end{array} \right] = \left[ \begin{array}{c}
\vc{0|0} \\
d_{1|0}
\end{array} \right].
\]
Then, $q_1(\vl{0})$ is represented by $\mxr{0}$ and $\vc{0} = \vc{0|0} - \tsz{1}\vr{}$.

\subsection*{Triangularization of Stochastic Equations}

Here, we detail an algorithm for transforming a system such as \eqp{stochsys-1} into the form \eqp{stochsys-2}. The material is from \cite[ch.~12, Appendix]{Hyndman:08}, with an additional fix: working with standard deviation instead of variance for numerical stability.

We are given a system of stochastic equations $q_i = (\va{i}, \mu_i, v_i)$, with the semantics
\[
\va{i}^T\vx{} = c_i \sim N(\mu_i, v_i),
\]
where the $c_i$ of different equations are independent. We also write $c_i = \mu_i + u_i$, where $u_i\sim N(0, v_i)$ are independent.

What we do is a clever variant of Gaussian elimination, which also keeps the right hand sides independent. The basic primitive is as follows. The input are stochastic equations $q_1$, $q_i$ with $a_{1 1} = 1$. The first equation is the pivot row. The goal is a transformation to $q_1^+$, $q_i^+$, such that $a_{i 1}^+ = 0$, but $a_{1 1}^+ = 1$. By Gaussian elimination:
\[
q_i^+ = q_i - a_{i 1} q_1,\quad v_i^+ = v_i + a_{i 1}^2 v_1.
\]
This transformation alone leaves $c_1$ and $c_i^+$ dependent, we also need
\[
q_1^+ = w_1 q_1 + w_i q_i,\quad w_1 = \frac{v_i}{v_i^+},\; w_i =
\frac{a_{i 1} v_1}{v_i^+}.
\]
Then, $a_{1 1}^+ = 1$, and
\[
\Ex\left[ u_i^+ u_1^+ \right] = \Ex\left[ (u_i - a_{i 1} u_1) (w_1 u_1 + w_i u_i)
\right] = w_i v_i - a_{i 1} w_1 v_1 = 0.
\]
Importantly, this works also if $v_i = 0$ (we assume that $a_{i 1}\ne 0$, otherwise nothing has to be done). 
Note that if $v_1 = 0$, we use $w_1 = 1$, $w_i = 0$, even if $v_i = 0$. Namely, if both variances are zero, the equations are deterministic, and Gaussian eliminiation alone is sufficient. It is interesting to note that $v_i = 0$ implies $v_1^+ = 0$.

In our implementation, we represent standard deviations $s_i = v_i^{1/2}$ instead of variances, as detailed at the end of this section. This should improve numerical robustness in the presence of small, but nonzero variances.

Suppose we start with the stochastic system
\begin{equation}\label{eq:stochsys-3}
\left[ \begin{array}{cc}
\mxr{} & \mxzero \\
\mxa{} & \mxb{}
\end{array} \right]\quad \left[ \begin{array}{c}
\vc{} \\
\vd{}
\end{array} \right],
\end{equation}
where $\vc{}$, $\vd{}$ consists of independent Gaussians, and $\mxr{}\in\R^{d\times d}$ is unit upper triangular. $\mxa{}$ may be sparse, which we exploit. In a first part, the primitive is applied to obtain
\begin{equation}\label{eq:stochsys-4}
\left[ \begin{array}{cc}
* & * \\
\mxzero{} & \mxb{}'
\end{array} \right]\quad \left[ \begin{array}{c}
* \\
\vd{}'
\end{array} \right],
\end{equation}
In the second part, $[\mxb{}' | \vd{}']$ is transformed to $[\mxr{}' | \vc{}']$. Here, we iterate between row permutations for pivoting, division by the pivot element, and the primitive to zero out the column below. The pivot is selected to maximize the absolute value in the system matrix column, as is normally done in the QR decomposition.

In our implementation, the whole matrix in \eqp{stochsys-3} is passed as dense matrix, and is overwritten by all of \eqp{stochsys-4}, including the ``*'' parts. Also, the complete right hand side is returned.

Finally, here is how to use standard deviations $s_1$, $s_i$ instead of variances $v_1$, $v_i$. If $s_1 = 0$ or $a_{i 1} = 0$, nothing has to be done, in that $w_1 = 1$, $w_i = 0$. Assume that $s_1 > 0$ and that $a_{i 1} \ne 0$. The goal is not to compute variances at all, as this could trigger overflow or underflow. Recall that in exact arithmetic:
\[
w_1 = \frac{s_i^2}{v_i^+},\quad w_i = \frac{(a_{i 1} s_1) s_1}{v_i^+},\quad
v_i^+ = s_i^2 + (a_{i 1} s_1)^2.
\]
We make a case distinction. First, suppose that $s_i > |a_{i 1} s_1|$. Then,
\[
w_1 = \frac{1}{1 + \alpha^2},\quad w_i = \frac{\alpha (s_1/s_i)}{1 + \alpha^2},
\quad \alpha = \frac{a_{i 1} s_1}{s_i}.
\]
We directly compute
\[
s_i^+ = s_i\sqrt{ 1 + \alpha^2 },\quad s_1^+ = w_1 s_1 \sqrt{ 1 + \alpha^2 }
= \frac{s_1}{\sqrt{ 1 + \alpha^2 }}.
\]
This is because
\[
\frac{w_i s_i}{w_1 s_1} = \frac{a_{i 1} s_1^2 s_i}{s_i^2 s_1} = \alpha.
\]
Second, if $s_i \le |a_{i 1} s_1|$. Then,
\[
w_1 = \frac{\beta^2}{1 + \beta^2},\quad w_i = \frac{1/a_{i 1}}{1 + \beta^2},
\quad \beta = \frac{s_i}{a_{i 1} s_1}.
\]
We directly compute:
\[
s_i^+ = |a_{i 1} s_1| \sqrt{ 1 + \beta^2 },\quad s_1^+ = |w_i| s_i \sqrt{ 1 +
	\beta^2 } = \frac{s_i/|a_{i 1}|}{\sqrt{ 1 + \beta^2 }} =
\frac{|\beta| s_1}{\sqrt{ 1 + \beta^2 }}
\]
Note that in both cases, we have:
\[
s_1^+ = \sqrt{w_1} s_1,\quad s_i^+ = s_i / \sqrt{w_1}.
\]

\section{Inference over Weights $\vw{}$}
\label{app:InfOverW}

We can extend this method in order to do inference over the weights $\vw{}$, instead of estimating them as part of $\vth{}$. Namely, we define
\[
q_t(\vl{t},\eps_t,\vl{t-1},\vw{}) = \delta(\vl{t}|\vl{t-1},\eps_t)
N(\eps_t|0,1) q_t(\vl{t-1},\vw{}),
\]
where the latter factor is a joint Gaussian. Moreover, $N(\tilde{z}_t | \vl{t-1})$ is replaced by $N(\tilde{z}_t | \vl{t-1},\vw{})$, since
\[
y_t = \va{t}^T\vl{t-1} + \vx{t}^T\vw{}.
\]

\subsubsection*{Forward Pass}

Some applications of joint inference require a forward pass only. For example, the posterior $P(\vw{} | \mathcal{D})$ can be determined this way. In the forward pass, we have
\begin{equation}\label{eq:joint-repr-qt}
\left[ \begin{array}{cc}
\mxr[l]{t-1} & \mxr[l,w]{t-1} \\
\mxzero & \mxr[w]{t-1}
\end{array} \right] \left[ \begin{array}{c}
\vl{t-1} \\
\vw{}
\end{array} \right] = \left[ \begin{array}{c}
\vc[l]{t-1} \\
\vc[w]{t-1}
\end{array} \right]
\end{equation}
as representation of $q_{t}(\vl{t-1}, \vw{})$. The forward system \eqp{stochsys-1} becomes:
\[
\left[ \begin{array}{cccc}
\mxr[l]{t-1} & \mxzero & \mxr[l,w]{t-1} & \vzero \\
-\mxq{t}\mxf{} & \mxq{t} & \mxzero & \vzero \\
\mxzero & \mxzero & \mxr[w]{t-1} & \vzero \\
\vzero^T & -\va{t+1}^T & -\vx{t+1}^T & 1
\end{array}\right] \left[ \begin{array}{c}
\vl{t-1} \\
\vl{t|t-1} \\
\vw{} \\
\tsz{t+1|t}
\end{array}\right] = \left[ \begin{array}{c}
\vc[l]{t-1} \\
\eps_t\|\vg{t}\|\vdelta{1} \\
\vc[w]{t-1} \\
\nu_{t+1}
\end{array}\right].
\]
Even though this system is quite a bit larger, the new parts are very sparse, a fact which the triangularization algorithm makes use of. After triangularization, we have
\[
\left[ \begin{array}{cccc}
* & * & * & * \\
\mxzero & \mxr[l]{t} & \mxr[l,w]{t} & \vr[l]{} \\
\mxzero & \mxzero & \mxr[w]{t} & \vr[w]{} \\
\vzero^T & \vzero^T & \vzero^T & 1
\end{array}\right] \left[ \begin{array}{c}
\vl{t-1} \\
\vl{t|t-1} \\
\vw{} \\
\tsz{t+1|t}
\end{array}\right] = \left[ \begin{array}{c}
* \\
\vc[l]{t|t-1} \\
\vc[w]{t|t-1} \\
d_{t+1|t}
\end{array}\right],
\]
from which $q_{t+1}(\vl{t}, \vw{})$ can be read off. Finally, the initial system for computing $q_1(\vl{0},\vw{})$ is given by
\[
\left[ \begin{array}{ccc}
\Id & \mxzero & \vzero \\
\mxzero & \Id & \vzero \\
-\va{1}^T & -\vx{1}^T & 1
\end{array}\right] \left[ \begin{array}{c}
\vl{0} \\
\vw{} \\
\tsz{1|0}
\end{array}\right] = \left[ \begin{array}{c}
\tvc[l]{0} \\
\tvc[w]{0} \\
\nu_{1}
\end{array}\right].
\]
Here, we assume that $\vw{}$ has a Gaussian prior with diagonal covariance, giving rise to the stochastic representation $\vw{} = \tvc[w]{0}$ for $P(\vw{})$.

The posterior $P(\vw{} | \mathcal{D})$ is obtained after the forward pass already, by marginalization of $q_T(\vl{T-1}, \vw{}) = P(\vl{T-1}, \vw{} | \mathcal{D})$. In fact, its representation is given by $(\mxr[w]{T-1}, \vc[w]{T-1})$.

\subsubsection*{Running Time}

What is the running time? Suppose that $\vl{t}\in\R^d$, $\vw{}\in\R^w$. Also, suppose that all $T$ positions are observed. We have to triangularize $T$ forward pass systems of size $2 d + w + 1$. This is done using row operations which cost $r(2 d + w)$. How many of them do we need? First, we annihilate $-\mxq{t}\mxf{}$, which costs $\text{nnz}(\mxq{t}\mxf{})\le d^2$ row operations. Only the two upper rows blocks are touched. This replaces $\mxq{t}\to \mxt{1}\in\R^{d\times d}$. Next, we triangularize $\mxt{1}$ and annihilate $-\va{t+1}^T$, which costs $\text{nnz}(\mxt{1}) + \text{nnz}(\va{t+1})\le d^2 + d$. Again, the third row block of size $w$ is not touched. The resulting matrix is
\[
\left[ \begin{array}{cccc}
* & * & * & \vzero \\
\mxzero & \mxr[l]{t} & \mxr[l,w]{t} & \vr[l]{} \\
\mxzero & \mxzero & \mxr[w]{t-1} & \vzero \\
\vzero^T & \vzero^T & \vt{2}^T & *
\end{array}\right].
\]
Now, $\mxr[w]{t-1}$ is upper triangular already, so we just have to annihilate $\vt{2}^T$, at the cost of $\text{nnz}(\vt{2})\le w$ row operations. The total number of row operations is
\[
\text{nnz}(\mxq{t}\mxf{}) + \text{nnz}(\mxt{1}) + \text{nnz}(\va{t+1}) +
\text{nnz}(\vt{2}) + 1.
\]
All in all, the costs are about $2 d^2 + w$ row operations, therefore $O(w^2 + w d^2 + d^3)$. The scaling is cubic in $d$, but only quadratic in $w$.

\subsubsection*{Backward Pass}

Suppose now that we wish to do joint smoothing over latent states and weights $\vw{}$. For a feature vector of small dimensions, this can be a viable alternative to learning weights by maximum likelihood.

First, we note that
\[
r_t(\vl{t},\eps_t,\vw{}) = q_t(\eps_t | \vl{t}, \vw{}) r_t(\vl{t},\vw{}),
\]
the latter a joint Gaussian. First, we need to determine a representation of $q_t(\eps_t | \vl{t}, \vw{})$ which is computed during the forward pass. Recall the representation of $q_t(\vl{t-1},\vw{})$ from \eqp{joint-repr-qt}. In particular, this implies that $q_t(\vl{t-1} | \vw{})$ is represented by the stochastic system
\[
\mxr[l]{t-1}\vl{t-1} = \vc[l]{t-1} - \mxr[l,w]{t-1} \vw{}.
\]
Note that only the mean on the right hand side is modified, the variances are those of $\vc[l]{t-1}$. Next, it is easy to see that $q_t(\vl{t}, \vw{} | \eps_t) = q_t(\vl{t} | \vw{},\eps_t) q_t(\vw{})$, so that
\[
q_t(\vl{t},\eps_t | \vw{}) \propto q_t(\vl{t} | \vw{},\eps_t) N(\eps_t | 0,1).
\]
This means that we can determine $q_t(\eps_t | \vl{t},\vw{})$ without using the marginal $q_t(\vw{})$. We will modify the derivation from above. First, $q_t(\vl{t-1} | \vw{})$ replaces $q_t(\vl{t-1})$. We simplify the notation by writing $\mxr{t-1}$ instead of $\mxr[l]{t-1}$, $\vm{t-1}$ for $\vm[l]{t-1}$, $\vv{t-1}$ for $\vv[l]{t-1}$, and $\mxs{t-1}$ instead of $\mxr[l,w]{t-1}$. Mean (conditional) and covariance of $q_t(\vl{t-1} | \vw{})$ are
\[
\mxsigma{t-1} = \mxr{t-1}^{-1}\mxv{t-1}\mxr{t-1}^{-T},\quad \vmu{t-1} =
\mxr{t-1}^{-1}\left( \vm{t-1} - \mxs{t-1}\vw{} \right).
\]
We can represent $q_t(\eps_t | \vl{t}, \vw{})$ as
\[
\eps_t = \vk{t}^T\vl{t} + \vq{t}^T\vw{} + \nu_t^{\eps},\quad \nu_t^{\eps}\sim
N(m_t^{\eps}, v_t^{\eps}).
\]
If we repeat the derivation above, we end up with
\[
v_t^{\eps} = \left( 1 + \tvh{}{}^T\mxv{t-1}^{-1}\tvh{}
\right)^{-1},\quad \tvh{} = \mxr{t-1}\mxf{}^{-1}\vg{t}
\]
and
\[
\vk{t} = v_t^{\eps} \mxf{}^{-T}\mxr{t-1}^T\mxv{t-1}^{-1}\tvh{},\quad
m_t^{\eps} = -v_t^{\eps} \vm{t-1}^T\mxv{t-1}^{-1}\tvh{},
\]
as well as
\[
\vq{t} = v_t^{\eps} \mxs{t-1}^T\mxv{t-1}^{-1}\tvh{}.
\]
We need to store $\vq{t}$ alongside $\vk{t}$, $m_t^{\eps}$, $v_t^{\eps}$.

Next, we need to propagate $r_t(\vl{t},\vw{})$ in the backward pass. Suppose that $r_t(\vl{t},\vw{})$ is represented as
\[
\left[ \begin{array}{cc}
\mxr[l]{t} & \mxr[l,w]{t} \\
\mxzero & \mxr[w]{t}
\end{array} \right] \left[ \begin{array}{c}
\vl{t} \\
\vw{}
\end{array} \right] = \left[ \begin{array}{c}
\vc[l]{t} \\
\vc[w]{t}
\end{array} \right]
\]
The relevant equation for the backward pass is
\[
\mxf{}\vl{t-1} = \vl{t} - \vg{t}\eps_t = \left( \Id - \vg{t}\vk{t}^T \right)
\vl{t} - \vg{t}\vq{t}^T\vw{} - \vg{t}\nu_t^{\eps}.
\]
Multiplying with $\mxq{t}$ (where $\mxq{t}\vg{t} = \|\vg{t}\|\vdelta{1}$), we obtain the backward pass stochastic system
\[
\left[ \begin{array}{ccc}
\mxr[l]{t} & \mxzero & \mxr[l,w]{t} \\
\mxq{t}(\Id - \vg{t}\vk{t}^T) & -\mxq{t}\mxf{} & -\|\vg{t}\|\vdelta{1}\vq{t}^T \\
\mxzero & \mxzero & \mxr[w]{t}
\end{array}\right] \left[ \begin{array}{c}
\vl{t} \\
\vl{t-1} \\
\vw{} \\
\end{array}\right] = \left[ \begin{array}{c}
\vc[l]{t} \\
\nu_t^{\eps}\|\vg{t}\|\vdelta{1} \\
\vc[w]{t}
\end{array}\right].
\]
After triangularization, we can read off $r_{t-1}(\vl{t-1},\vw{})$. Note that the lowest block of rows of the system is never touched: $r_t(\vw{}) = r_{T-1}(\vw{})$ for all $t$. This makes sense, since the posterior of $\vw{}$ is obtained after the forward pass already. This fact could be used to save some space when returning the posteriors $P(\vl{T-1},\vw{} | \mathcal{D})$ and $P(\vl{0},\vw{} | \mathcal{D})$, since $\mxr[w]{T-1} = \mxr[w]{0}$. We do not currently do this.

At this point, we have
\[
\Ex[\eps_t | \mathcal{D}] = \vk{t}^T\Ex[\vl{t}|\mathcal{D}] +
\vq{t}^T\Ex[\vw{}|\mathcal{D}] + m_t^{\eps}
\]
and
\[
\Ex[y_t | \mathcal{D}] = \tva{t}^T\vm{t-1} + b_t,\quad
\Var[y_t | \mathcal{D}] = \tva{t}^T(\diag\vv{t-1})\tva{t},\quad
\tva{t} = \mxr{t-1}^{-T} [\va{t}^T, \vx{t}^T]^T,
\]
where
\[
\mxr{t} =  \left[ \begin{array}{cc}
\mxr[l]{t} & \mxr[l,w]{t} \\
\mxzero & \mxr[w]{t}
\end{array} \right],\quad \vm{t} =  \left[ \begin{array}{c}
\vm[l]{t} \\
\vm[w]{t}
\end{array} \right], \quad \vv{t} =  \left[ \begin{array}{c}
\vv[l]{t} \\
\vv[w]{t}
\end{array} \right].
\]
Since we need to store the $\vq{t}$ vectors alongside $\vk{t}$, the storage requirements for Kalman smoothing are increased.

\section{Missing Observations}
\label{app:MissingObs}
Suppose now that some observations are missing. This extension is required for a number of reasons. Here, we detail changes to the mode finding process.

First, here is how the Gaussian posterior inference procedure is modified. Suppose that $(\tilde{z}_t, \sigma_t^2)$ are missing for some $t$. Note that only the forward pass depends on observations: computations during the backward pass are unchanged. First, if $t = 1$ (first observation), we simply start with $q_1(\vl{0}) = P(\vl{0})$, and do not add to the log likelihood. Now, suppose that $t>1$. The forward pass system becomes

\[
\left[ \begin{array}{ccc}
\mxr{t-1} & \mxzero \\
-\mxq{t}\mxf{} & \mxq{t}
\end{array} \right] \left[ \begin{array}{c}
\vl{t-1} \\
\vl{t|t-1}
\end{array} \right] = \left[ \begin{array}{c}
\vc{t-1} \\
\eps_t\|\vg{t}\|\vdelta{1}
\end{array} \right].
\]
After triangularization, $\mxr{t}$ and $\vc{t}$ are read off directly. The log likelihood term is not added to. Everything else remains the same.

Note that for simplicity, our Gaussian inference code assumes that $b_t = 0$ for all $t$. The calling code replaces $\tilde{z}_t$ with $\tilde{z}_t - b_t$. Also, $b_t$ is added back to $\Ex[y_t | {\cal D}]$ later on. It is important that this back-addition also happens for missing locations.

Next, here is how to modify the mode finding procedure. First, all computations involving the likelihood potentials are done only on observed positions. It is simplest to use full vectors, with the understanding that unobserved positions are undefined. In particular, $[\tilde{z}_t]$ and $[\sigma_t]$ are of this kind, as well as the targets $[z_t]$ itself. Next, the line search is affected as follows. Denote the index of observed positions by $O$. The LS criterion is $F(\alpha) = F(\vs{} + \alpha\vd{})$, where
\[
F(\vs{}) = -\log P(\vz{O} | \vy{O}) - \log P(\vs{}),\quad \vy{} = \mxm{}\vs{}
+ \vb{}.
\]
So the log likelihood is summed over $O$ only. Moreover, the first part of $F'(0)$ is given by $(\mxm{}\vd{})_O^T(\nabla_{\vy{O}}-\log P(\vz{O} | \vy{O}))$. Everything else stays the same.

\section{Further Details on Optimization}
\label{app:LineSearch}
Here we collect more details for the Newton-Raphson algorithm.
\subsection*{Line Search}

The Newton/Raphson algorithm is not convergent from any starting point without a line search. The problem is that the full Newton step can overshoot. A line search is easy to do. Recall that $\vy{} = \mxm{}\vs{} + \vb{}$, where $\vs{} = [\veps{0}^T, \vl{0}^T]^T$. Suppose we compute a full Newton step $\vs{}\mapsto \vs{}'$. The desired update would be to $\vs{\alpha} = \vs{} + \alpha{}\vd{}$, where $\vd{} = \vs{}' - \vs{}$ and $\alpha\in (0,1]$. Note that
\[
\vy{\alpha} = \mxm{}\vs{\alpha} + \vb{} = \vy{} + \alpha\mxm{}\vd{}.
\]
We precompute $\mxm{}\vd{}$. The subsequent backtracking line search has negligible cost. The line search criterion is
\[
F(\alpha) := F(\vs{} + \alpha\vd{}),\quad \alpha > 0.
\]
We also need the derivative $F'(0)$ at $\alpha=0$. Recall that
\[
F(\vs{}) = -\log P(\vz{}|\vy{}) - \log P(\vs{}),\quad \vy{} = \mxm{}\vs{} +
\vb{},
\]
where $P(\vs{})$ is Gaussian. Therefore:
\[
F'(0) = \vd{}^T(\nabla_{\vs{}} F) = (\mxm{}\vd{})^T(\nabla_{\vy{}}
-\log P(\vz{}|\vy{})) + \vd{}^T(\nabla_{\vs{}} -\log P(\vs{})).
\]
Here, $\nabla_{\vy{}}-\log P(\vz{}|\vy{}))$ is computed alongside the evaluation of $F(\vs{})$. Also, $P(\veps{}) = N(\vzero,\Id)$, so that $\vd{\eps}^T(\nabla_{\veps{}} -\log P(\veps{})) = \vd{\eps}^T\veps{}$. Finally, suppose that $P(\vl{0})$ is represented by $\mxr{0}\vl{0} \sim N(\vm{0},\mxv{0})$. Then,
\[
\vd{l_0}^T(\nabla_{\vl{0}} -\log P(\vl{0})) = \vd{l_0}^T\mxsigma{0}^{-1}\left(
\vl{0} - \vmu{0} \right) = (\mxr{0}\vd{l_0})^T\mxv{0}^{-1}\left( \mxr{0}\vl{0}
- \vm{0} \right).
\]
Importantly, apart from the computation of $\mxm{}\vd{}$, the line search comes essentially for free.

\subsection*{Initialization for Inner Optimization}

When poorly initialized, the Newton algorithm may fail completely. Namely, when the initial $y_t$ are large, the Gaussian fits have infinite variances, and some have infinite means as well.

We employ two mechanisms. First, if only a minority of potentials are ``maxed out'' in this way, it makes sense to just ignore them temporarily. This is best done by allowing for infinite variances in the linear-Gaussian model. To this end, we could implement support for infinite variances in stochastic triangularization, as proposed by Snyder. Given that, whenever $\phi_t''(y_t)$ falls below a threshold, we use $\sigma_t = \infty$ and $\tsz{t} = 0$ (the latter is arbitrary). Another idea is to simply mark such positions $t$ as unobserved, as is already supported in the Kalman smoother code. This mechanism helps to avoid a numerical breakdown of Newton, as long as not many potentials are affected. On the other hand, if many potentials ``max out'', the likelihood may have to be chosen differently.

The second mechanism is a heuristic to select a sensible starting point. The key idea is to start Newton at a point where the inputs $y_t$ to likelihood potentials do not fall into tails, where second derivatives are very flat. The following is just a first idea, which could be improved. Recall that $\vy{} = \mxm{}\vs{} + \vb{}$. First, we select a value $\bar{y}$ such that $P(z | y=\bar{y})$ is equal to the maximum likelihood estimate from the training targets. If the target vector is too short or shows too little variation, we fall back on a default value estimated from all data. Given $\bar{y}$ and $\vb{}$, we have to choose $\vs{}$ such that $\vy{} = \mxm{}\vs{} + \vb{} \stackrel{!}{=} \bar{y}\vone$.

There are some restrictions on the ISSM parameters $\va{t}$, $\vg{t}$. First, we need $|\va{t+1}^T\vg{t}| > \eps$ for $t=\rng{T-1}$. Also, $\va{1}$ must have at least one entry bounded away from zero. As long as the ISSM has a level component, these assumptions always hold. First, we select $\vl{0}$. Let $j = \argmax_j |(\va{1})_j|$. Set $\vl{0} = ((\bar{y} - b_1) / (\va{1})_j) \vdelta{j}$. This ensures that $y_1 = \bar{y}$. Next, for $t = 1, 2, \dots$:
\[
\bar{y} = \va{t+1}^T\left( \mxf{}\vl{t-1} + \eps_t\vg{t} \right) + b_{t+1}\quad
\Rightarrow\quad \eps_t =
\frac{\bar{y} - b_{t+1} - \va{t+1}^T\mxf{}\vl{t-1}}{\va{t+1}^T\vg{t}}.
\]

\section{Outer Criterion and Gradient}
\label{app:Gradient}

Denote the posterior mode found in the first phase by $\hvs{} = [\hveps{}{}^T, \hvl{0}{}^T]^T$. The surrogate criterion for the negative log marginal likelihood is $\psi(\vth{}, \hvs{})$. Here, we show how to compute its value and gradient. In the sequel, we drop the conditioning on $\vth{}$ of all distributions from the notation. Recall that the true log likelihood is
\begin{equation}\label{eq:true-loglh}
\log P(\vz{}) = \log \int P(\vz{}|\vy{}) P(\vy{}|\veps{},\vl{0}) P(\veps{})
P(\vl{0})\, d\vy{} d\veps{} d\vl{0}.
\end{equation}
Here, $P(\vy{}|\vs{})$ is a delta distribution. Importantly, except
for the first non-Gaussian factor, this is exactly the linear-Gaussian
model for which posterior inference was developed above. Denote
$\hvy{} = \vy{}(\hvs{})$. Recall that $\phi_t(y_t) = -\log P(z_t |
y_t)$, and denote $\tsphi{t}(y_t) := -\log N(\tsz{t} | y_t,
\sigma_t^2)$, where $\tsz{t}$ and $\sigma_t^2$ are fixed functions of
$\hvy{}$, namely
\[
\sigma_t^2 = \frac{1}{\phi_t(\hsy{t})''},\quad \tsz{t} = \hsy{t} -
\sigma_t^2 \phi_t(\hsy{t})'.
\]
We have to be explicit now about the constants:
\[
\phi_t(y_t) \approx \tsphi{t}(y_t) + \gamma_t,\quad \gamma_t :=
\phi_t(\hsy{t}) - \tsphi{t}(\hsy{t}).
\]
The central observation is this. The Laplace approximation criterion is obtained by plugging $e^{-\tsphi{t}(y_t) - \gamma_t}$ in place of $e^{-\phi_t(y_t)}$ into \eqp{true-loglh}:
\[
\psi = -\log\int \left( \prod_t e^{-\gamma_t} N(\tsz{t} | y_t, \sigma_t^2)
\right) P(\vy{}|\veps{},\vl{0}) P(\veps{})
P(\vl{0})\, d\vy{} d\veps{} d\vl{0}.
\]
Therefore, if $P(\tvz{})$ denotes the (marginal) likelihood of the linear-Gaussian model obtained for fixed $\tsz{t}$, $\sigma_t^2$, we have that
\[
\psi = \sum_t\gamma_t - \log P(\tvz{}).
\]
This means that a single forward pass of the information filter detailed above is sufficient for computing $\psi$. More precisely:
\[
\log P(\tvz{}) = \sum_t \log N(\tsz{t} | m_{t|t-1}, v_{t|t-1}),
\]
where $m_{t+1|t}, v_{t+1|t}$ are mean and variance of the $d_{t+1|t}$ right hand side in \eqp{stochsys-2}.

The gradient computation is challenging. $\psi$ depends on $\vth{}$ through various channels. Denote $\mxxi{} = \{\hsy{t}, \tsz{t}, \sigma_t^2\}$. We have that
\[
\psi = \psi(\vth{},\mxxi{}),\quad \mxxi{} = \mxxi{}(\vth{},\hvs{}),\quad
\hvs{} = [\hveps{}{}^T, \hvl{0}{}^T]^T = \hvs{}(\vth{}).
\]
Namely, (a) directly for fixed $\mxxi{}$, (b) through $\mxxi{}$ for fixed $\hvs{}$, and (c) through $\hvs{} = \hvs{}(\vth{})$. It is important to note that for {\em other} commonly used variational inference techniques, such as expectation propagation, the equivalent approximation to $-\log P(\vz{})$ is such that both (b) and (c) can be ignored. This is not true for the Laplace approximation.

We will deal with each channel separately. Below, we will also further split $\mxxi{}$ into $\hvy{}$ and the rest. Then, we have
\[
\nabla_{\vth{}}\psi = \partial_{\vth{}}\psi + \partial_{\hvs{}}\psi
\frac{\partial\hvs{}}{\partial\vth{}}, \quad
\partial_{\vth{}}\psi(\vth{},\hvs{}) = \partial_{\vth{}}\psi(\vth{},\hvy{}) +
\sum_t \partial_{\hsy{t}}\psi \cdot \partial_{\vth{}}\hsy{t}.
\]

\subsubsection*{Gradient: Channel (a)}

Recall that $\mxxi{} = \{\hsy{t}, \tsz{t}, \sigma_t^2\}$ is considered fixed. In this case, $\psi(\vth{}) \doteq -\log P(\tvz{})$ (except for likelihood parameters: see below), where the Gaussian likelihood potentials are fixed. We are now dealing with the exact log partition function of a linear-Gaussian model, and can use the well-known relationship
\[
\nabla_{\vth{}}\psi = \Ex\left[ \vkappa{}(\veps{}, \vl{0}) \right],
\]
where $\Ex[\cdot]$ is over the approximate Gaussian posterior obtained as second order approximation at the mode. In particular, there is no dependence on the non-Gaussian likelihood potentials. In the sequel, $\Ex[\cdot]$ denotes expectation over this Gaussian posterior, unless otherwise said.

The parameters $\vth{}$ appear in several places. First, the weights $\vw{}$ come in via $b_t = \vx{t}^T\vw{}$ and $y_t = \tilde{y}_t + b_t$. Second, $\vth{\text{ES}}$ is concentrated in $\vg{t}$. Third, the prior $P(\vl{0})$ may depend on $\vth{}$.

We currently lack a computationally efficient way to compute the contribution w.r.t.\ $\vth{\text{ES}}$. The solution for now is to estimate this contribution by finite differences. This is not so bad, because $\vth{\text{ES}}$ is small, and the finite difference evaluations are done for fixed $\hvy{}$.


For the derivative w.r.t.\ $b_t$, recall that we can consider the likelihood potential $N(\tsz{t} | y_t, \sigma_t^2)$ fixed. Now,
\[
\partial_{b_t} -\log N(\tsz{t} | y_t, \sigma_t^2) = \sigma_t^{-2}(y_t - \tsz{t}),
\]
so that
\[
\partial_{b_t}\psi = \sigma_t^{-2}\left( \Ex[y_t] - \tsz{t} \right).
\]

For the $P(\vl{0})$ parameters, we have that
\[
d\psi = \Ex\left[ d -\log P(\vl{0}) \right].
\]
Anything further depends on the parameterization of $P(\vl{0})$. For example, assume that $P(\vl{0}) = N(\vmu{0},\mxv{0})$. Then:
\[
d -\log P(\vl{0}) = -\ve{0}^T(d\vmu{0}) + \frac{1}2\trace\left( \mxv{0}^{-1}
- \ve{0}\ve{0}^T \right) d(\mxv{0}),\quad \ve{0} := \mxv{0}^{-1}(\vl{0} -
\vmu{0}),
\]
therefore
\[
\partial_{\vmu{0}}\psi = \mxv{0}^{-1}\left( \vmu{0} - \Ex[\vl{0}]
\right) =: -\Ex[\ve{0}].
\]
Suppose that $\mxv{0}=\diag\vv{0}$ is diagonal. Then,
\[
\partial_{\vv{0}}\psi = \frac{1}2\left( \vv{0}^{-1} - \Ex[\ve{0}]^2 -
\vv{0}^{-2}\circ\diag^{-1}(\Cov[\vl{0}]) \right).
\]

Finally, $\vth{}$ may contain parameters of the likelihood potentials $\phi_t(\cdot)$. Since $\hvy{}$ and $\mxxi{}$ are fixed, we have that $\psi \doteq \sum_t \phi_t(\hsy{t})$ for these parameters, therefore
\[
\partial_{\vth{\text{LH}}}\psi = \sum_t \partial_{\vth{\text{LH}}}\phi_t(\hsy{t}).
\]

\subsubsection*{Gradient: Channel (b)}

Here, we consider the channel through $\mxxi{} = \{\hsy{t}, \tsz{t}, \sigma_t^2\}$ for fixed mode $\hvs{}$. We decompose $\mxxi{}$ into $\hvy{}$ and the rest. There are two parts to the (b) gradient:
\begin{itemize}
	\item
	Channel through $\{\tsz{t}, \sigma_t^2\}$ for fixed $\hvy{}$
	\item
	Channel through $\hvy{}$
\end{itemize}

To deal with the first part, consider $\hvy{}$ fixed. This contribution arises only if $\vth{\text{LH}}$ is not empty. Let $\theta$ be a component of $\vth{\text{LH}}$. The expressions have the form $\sum_t$. We focus on a summand for fixed $t$, dropping it from the notation for now. We need $\partial_{\theta}\tsz{}$, $\partial_{\theta}\sigma^2$. There are two parts. First, $\gamma\doteq \log N(\tsz{} | \hsy{},\sigma^2)$, and
\[
\partial_{\tsz{}}\gamma = -\xi(\hsy{}),\quad \partial_{\sigma^2}\gamma =
-\frac{1}2\left( \sigma^{-2} - \xi(\hsy{})^2 \right),\quad
\xi(y) := \sigma^{-2} (\tsz{} - y).
\]
And
\[
\partial_{\theta}\gamma = \partial_{\tsz{}}\gamma\cdot \partial_{\theta}\tsz{}
+ \partial_{\sigma^2}\gamma\cdot \partial_{\theta}\sigma^2.
\]
Next, $-\log P(\tvz{})$. We need $\partial_{\tsz{}}\tsphi{}(y)$ and $\partial_{\sigma^2}\tsphi{}(y)$, then
\[
\partial_{\theta}(-\log P(\tvz{})) = \Ex\left[ \partial_{\tsz{}}\tsphi{}(y)
\right] \partial_{\theta}\tsz{} + \Ex\left[ \partial_{\sigma^2}\tsphi{}(y)
\right] \partial_{\theta}\sigma^2,
\]
We can use that $\gamma = -\tsphi{}(\hsy{})$.
\[
\begin{split}
\Ex\left[ \partial_{\tsz{}}\tsphi{}(y) \right] & = \Ex[\xi(y)] = \sigma^{-2}
\left( \tsz{} - \Ex[y] \right), \\
\Ex\left[ \partial_{\sigma^2}\tsphi{}(y) \right] & = \frac{1}2\left(
\sigma^{-2} - \Ex[\xi(y)]^2 - \Var[\xi(y)] \right) = \frac{1}2 \sigma^{-2}
\left( 1 - \sigma^{-2}\left( (\tsz{} - \Ex[y])^2 + \Var[y] \right) \right).
\end{split}
\]

We can compute two vectors $\ve[\tsz{}]{}$, $\ve[\sigma^2]{}$ in $\R^T$. Then, the gradient contributions are
\[
\partial_{\theta}\psi = (\ve[\tsz{}]{})^T [\sigma^{-2}\partial_{\theta}\tsz{t}] +
(\ve[\sigma^2]{})^T [\sigma^{-4}\partial_{\theta}\sigma^2].
\]
Here,
\[
e^{(\tsz{})} = \sigma^2\left( \Ex[\xi(y)] - \xi(\hsy{}) \right) = \hsy{} -
\Ex[y],\quad
e^{(\sigma^2)} = \frac{1}2\left( (\tsz{} - \hsy{})^2 -
(\tsz{} - \Ex[y])^2 - \Var[y] \right).
\]
Also,
\[
\begin{split}
\sigma^{-4}\partial_{\theta}\sigma^2 = -\partial_{\theta}\phi''(\hsy{}),\quad
\sigma^{-2}\partial_{\theta}\tsz{t} & = -\partial_{\theta}\phi'(\hsy{}) -
\phi'(\hsy{}) \sigma^2 (\sigma^{-4}\partial_{\theta}\sigma^2) =
-\partial_{\theta}\phi'(\hsy{}) + \phi'(\hsy{}) \sigma^2
\partial_{\theta}\phi''(\hsy{}) \\
& = -\partial_{\theta}\phi'(\hsy{}) -
(\tsz{} - \hsy{}) \partial_{\theta}\phi''(\hsy{}).
\end{split}
\]
Therefore, the contribution to $\partial_{\theta}\psi$ is
\[
-e^{(\tsz{})} \partial_{\theta}\phi'(\hsy{}) -
\left( e^{(\sigma^2)} + (\tsz{} - \hsy{}) e^{(\tsz{})} \right)
\partial_{\theta}\phi''(\hsy{}).
\]
We precompute two vectors $\ve[1]{}$, $\ve[2]{}$, then:
\[
\partial_{\theta}\psi = (\ve[1]{})^T [\partial_{\theta}\phi_t'(\hsy{t})] +
(\ve[2]{})^T [\partial_{\theta}\phi_t''(\hsy{t})].
\]
Some algebra gives
\[
e^{(1)} = \Ex[y] - \hsy{},\quad e^{(2)} = \frac{1}2\left( (\Ex[y] - \hsy{})^2
+ \Var[y] \right).
\]

Next, we consider the channel through $\hvy{}$ for fixed $\hvs{}$, which is given by
\[
\sum_t \partial_{\hsy{t}}\psi \cdot \partial_{\vth{}}\hsy{t}
\]

{\bf Computation of $\partial_{\hsy{t}}\psi$}. Our derivation does not assume that $\hvs{}$ is the posterior mode for the current $\vth{}$. In fact, $\hvs{}$ can be anything. Fix some $t$ and drop it from notation. Our goal is $\partial_{\hsy{}}\psi$. Denote $e_k := \phi^{(k)}(\hsy{})$. Recall that $\psi = \gamma - \log P(\tvz{})$. First,
\[
\tsphi{}(y) \doteq \frac{1}2 e_2 (\tsz{} - y)^2 - \frac{1}2 \log e_2,\quad
\tsz{} = \hsy{} - \frac{e_1}{e_2}
\]
and
\[
\partial_{\hsy{}}\tsphi{}(y) = \frac{e_3}2 (\tsz{} - y)^2 + e_2
(\partial_{\hsy{}}\tsz{}) (\tsz{} - y) - \frac{e_3}{2 e_2} =: q(y),\quad
e_2 \partial_{\hsy{}}\tsz{} = \frac{e_1 e_3}{e_2}.
\]
In the sequel, $\Ex[\cdot]$ is short for $\Ex[\cdot | \mathcal{D}]$. Now:
\[
\partial_{\hsy{}}-\log P(\tvz{}) = \Ex\left[ \partial_{\hsy{}}\tsphi{}(y) \right]
= \Ex[q(y)] = q(\Ex[y]) + \frac{e_3}2 \Var[y].
\]
Also,
\[
\partial_{\hsy{}}\gamma = \partial_{\hsy{}}\left( \phi(\hsy{}) - \tsphi{}(\hsy{})
\right) = -[\partial_{\hsy{}}\tsphi{}(y)]_{y=\hsy{}} = -q(\hsy{}),
\]
so that
\[
\partial_{\hsy{}}\psi = q(\Ex[y]) - q(\hsy{}) + \frac{e_3}2 \Var[y].
\]
Some more algebra on the difference results in
\[
\partial_{\hsy{}}\psi = \frac{e_3}2\left( (\hsy{} - \Ex[y])^2 + \Var[y] \right).
\]

What happens if $\hsy{}$ is the posterior mode for $\vth{}$? Then, $\Ex[y] = \hsy{}$ (for all $t$), therefore $\partial_{\hsy{}}\psi = \frac{e_3}2\Var[y]$. This expression does not vanish. This is the reason why the gradient computation for Laplace is hard to do. Note that the third derivative $\phi_t'''(\hsy{t})$ is required here.

{\bf Computation of $\partial_{\vth{}}\hsy{t}$}. Recall that
\[
\hsy{t} = \va{t}^T\hvl{t-1} + b_t,\quad \hvl{t} = \mxf{}\hvl{t-1} + \vg{t}
\hseps{t},
\]
where $b_t = \vx{t}^T\vw{}$. Therefore,
\[
\partial_{\vw{}}\hsy{t} = \vx{t}.
\]
First, assume that $\vth{\text{ES}}$ is concentrated in $\vg{t}$. Let $\theta$ be a component in $\vth{\text{ES}}$. Then,
\[
\partial_{\theta}\hvl{t} = \mxf{} \partial_{\theta}\hvl{t-1} + \hseps{t}
\partial_{\theta}\vg{t},\quad \partial_{\theta}\hvl{0} = \vzero,\quad
\partial_{\theta}\hsy{t} = \va{t}^T \partial_{\theta}\hvl{t-1}.
\]
This is done with a simple forward iteration, using $\partial_{\theta}\vg{t}$ in place of $\vg{t}$. Note that if $\vth{}$ contains parameters of $P(\vl{0})$, these do not play a role here, since $\hvy{}$ does not depend on the prior (for fixed $\hvs{}$). The only service we need here is
\[
\vth{}\mapsto \partial_{\theta}\vg{t},\quad \theta\in \vth{\text{ES}}.
\]

More generally, $\vth{\text{ES}}$ may also parameterize $\va{t}$ and $\mxf{}$. For example, this happens if we use damping and wish to learn the damping parameter by maximum likelihood. Suppose that $\theta$ is a component of $\vth{\text{ES}}$ with potential dependence on all ISSM parameters. Then,
\[
\partial_{\theta}\hvl{t} = \mxf{}\partial_{\theta}\hvl{t-1} +
(\partial_{\theta}\mxf{}) \hvl{t-1} + \hseps{t}
\partial_{\theta}\vg{t},\quad \partial_{\theta}\hsy{t} =
\va{t}^T\partial_{\theta}\hvl{t-1} + (\partial_{\theta}\va{t})^T\hvl{t-1}.
\]

\subsubsection*{Gradient: Channel (c)}

If $\theta$ is a single component of $\vth{}$, the remaining part is given by
\[
\partial_{\hvs{}}\psi \cdot \partial_{\theta}\hvs{}.
\]
Recall that
\[
F(\vs{},\vth{}) = -\log P(\vz{},\vs{} | \vth{}),
\]
and that $\hvs{}=\hvs{}(\vth{})$ is the minimizer of $F$ for fixed $\vth{}$. Therefore, $\partial_{\hvs{}} F = \vzero$ for all $\vth{}$ and
\[
0 = \nabla_{\theta}\partial_{\hvs{}} F = \partial_{\hvs{},\theta} F + (\partial_{\hvs{},\hvs{}} F) \frac{\partial\hvs{}}{\partial\theta}\quad \Rightarrow\quad
\frac{\partial\hvs{}}{\partial\theta} = -\left( \partial_{\hvs{},\hvs{}} F
\right)^{-1} \partial_{\hvs{},\theta} F.
\]
It is of course out of the question to compute or solve a system with $\partial_{\hvs{},\hvs{}} F$. The key observation is that the expression for $\partial\hvs{}/\partial\theta$ constitutes a posterior mean computation in a Gaussian sequential model of the form discussed above. In fact, exactly the same trick allows us to compute Newton update steps in this way. Still, we cannot afford a posterior computation for every $\theta$. We really need to compute
\[
-(\partial_{\hvs{}}\psi)^T \left( \partial_{\hvs{},\hvs{}} F \right)^{-1}
\partial_{\hvs{},\theta} F.
\]
Recall that $\psi = \psi(\vth{},\hvy{})$, where $\hvy{} = \mxm{}\hvs{} + \vb{}$, so that $\partial_{\hvs{}}\psi = \mxm{}^T\partial_{\hvy{}}\psi$. The {\em single} posterior mean computation we need is
\[
\vxi{} := \left( \partial_{\hvs{},\hvs{}} F \right)^{-1} \mxm{}^T\partial_{\hvy{}}\psi,
\]
with which
\[
\partial_{\hvs{}}\psi \frac{\partial\hvs{}}{\partial\theta} = -\vxi{}^T
\partial_{\hvs{},\theta} F.
\]
More specifically, $\vxi{}$ is the mean of the Gaussian
\[
Q(\vs{})\propto \exp\left( -\frac{1}2 \vs{}^T\left( \mxm{}^T\left( \diag
[\phi_t''(\hsy{t})] \right)\mxm{} + \partial_{\hvs{},\hvs{}}(-\log P(\hvs{}))
\right) \vs{} + \vs{}^T\mxm{}^T \partial_{\hvy{}}\psi \right).
\]
We can write $Q(\vs{})\propto N(\tvz{} | \vy{}) \tilde{P}(\vs{})$, where $\vy{} = \mxm{}\vs{}$ (no $\vb{}$ term). Also, $\tilde{P}(\vs{})$ has the same covariance as $P(\vs{})$, but has zero mean (in particular, $\tilde{P}(\vl{0})\ne P(\vl{0})$ in general), and $N(\tvz{} | \vy{})$ is given by
\[
\sigma_t^2 = \frac{1}{\phi_t''(\hsy{t})},\quad \tsz{t} = \sigma_t^2
\partial_{\hsy{t}}\psi.
\]
Now, $\vxi{}$ is computed by a single posterior inference call. Next,
\[
\partial_{\hvs{}} F = \mxm{}^T\left[ \phi_t'(\hsy{t}) \right] + \left[
\begin{array}{c}
\hveps{} \\
\partial_{\hvl{0}} -\log P(\hvl{0})
\end{array} \right],
\]
where $\hvy{} = \mxm{}\hvs{} + \vb{}$. For the second part, if $\theta_l$ is a parameter of the prior $P(\vl{0})$, we need a service for computing $\partial_{\theta}\partial_{\vl{0}} -\log P(\vl{0})$ for given $\vl{0}$. Finally,
\[
\partial_{\theta} \mxm{}^T\left[ \phi_t'(\hsy{t}) \right] =
(\partial_{\theta} \mxm{})^T\left[ \phi_t'(\hsy{t}) \right] + \mxm{}^T
\left[ \phi_t''(\hsy{t}) \partial_{\theta}\hsy{t} \right].
\]
The computation of $\partial_{\vth{}}\hsy{t}$ is derived above, it makes use of the forward iteration. We really only need
\[
\vxi{}^T\left( \partial_{\theta} \mxm{}^T\left[ \phi_t'(\hsy{t}) \right]
\right) = \left( \partial_{\theta} \mxm{}\vxi{} \right)^T \left[
\phi_t'(\hsy{t}) \right] + (\mxm{}\vxi{})^T \left[ \phi_t''(\hsy{t})
\partial_{\theta}\hsy{t} \right].
\]
Here, $\mxm{}\vxi{}$ is computed using a forward pass (with $\vb{}=\vzero$). Moreover, $\partial_{\theta} \mxm{}\vxi{}$ is computed using the $\partial_{\vth{}}\hsy{t}$ code (only $\vth{\text{ES}}$ are relevant, since $\vb{}=\vzero$), feeding in $\vxi{}$ instead of $\hvs{}$.

The computation is folded together with the (b) part. We first compute $[\partial_{\hsy{t}}\psi]$, then $\vxi{}$ by posterior inference, and $\mxm{}\vxi{}$ by a forward pass. The contribution to $\nabla_{\vw{}}\psi$ becomes
\[
\mxx{}\left[\partial_{\hsy{t}}\psi - \phi_t''(\hsy{t}) (\mxm{}\vxi{})_t
\right].
\]
For the part from $\vth{\text{ES}}$, we loop over the components $\theta$. We compute $\partial_{\theta}\hvy{}$, using a modified forward pass with $\partial_{\theta}\vg{t}$ in place of $\vg{t}$, and $\hveps{}$. Then, we compute $\partial_{\theta}\mxm{}\vxi{}$, using the same modified forward pass with $\partial_{\theta}\vg{t}$, but $\vxi{\eps}$ in place of $\hveps{}$. The contribution to $\nabla_{\theta}\psi$ is
\[
\sum_t\left( \partial_{\theta}\hsy{t}\left( \partial_{\hsy{t}}\psi -
\phi_t''(\hsy{t}) (\mxm{}\vxi{})_t \right) - \phi_t'(\hsy{t})
(\partial_{\theta}\mxm{}\vxi{})_t \right).
\]
Finally, if $\theta$ is a parameter of the prior $P(\vl{0})$, the contribution to the gradient is
\[
-\vxi{l_0}^T\left( \partial_{\theta}\partial_{\hvl{0}} -\log P(\hvl{0}) \right).
\]

We also allow the likelihood potentials $\phi_t(\cdot)$ to be parameterized. Suppose that $\theta$ is such a parameter. The gradient contribution is
\[
-(\mxm{}\vxi{})^T\left[ \partial_{\theta} \phi_t'(\hsy{t}) \right].
\]

\subsubsection*{Third Derivatives}

The logistic function is $\lambda(y) = \log(1 + e^y)$, its derivative is $\sigma(y) = (1 + e^{-y})^{-1}$. Also, $\sigma(y)' = \sigma(y) \sigma(-y)$. For the logistic binary classification likelihood,
\[
\phi''(y) = \sigma(y) \sigma(-y),\quad \phi^{(3)} = \phi''(y)\left( 1 - 2
\sigma(y) \right).
\]

The Poisson likelihood with logistic rate function $\lambda(y) = \log(1 + e^y)$ has
\[
\phi(y) = \lambda - z\log\lambda + \log(z!),\quad \phi'(y) = \sigma - z\alpha,
\quad \alpha := \frac{\sigma}{\lambda}.
\]
Also,
\[
\phi''(y) = \phi'(y) (1-\sigma) + z\alpha^2,
\]
since $\alpha(y)' = \alpha(1 - \sigma - \alpha)$. Note that for large negative $y$, we have that $\alpha\approx 1-\sigma \approx 1$. Now,
\[
\phi^{(3)}(y) = (1-\sigma)\left( \phi'' - \sigma\phi' \right) + 2 z\alpha^2
( 1 - \sigma - \alpha),\quad \phi'' - \sigma\phi' = \phi'(1 - 2\sigma)
+ z\alpha^2.
\]

Finally, for the Poisson likelihood with exponential rate function $\lambda(y) = e^y$, we have $\phi^{(3)}(y) = \phi''(y) = e^y$.

\subsubsection*{Missing Observations}

Suppose now that only positions $t\in O$ are observed, others are missing. In this case, the criterion becomes
\[
\psi = \log\int P(\vz{O} | \vy{O}) P(\vy{O} | \vs{}) P(\vs{})\, d\vy{O}
d\vs{} = \sum_{t\in O}\gamma_t - \log P(\tvz{O}).
\]
The relevant variables are $\vth{}$, $\hvs{}$, $\hvy{O}$. In particular, $\hvy{}$ is replaced by $\hvy{O}$ here.

Note that both $F(\vs{})$ and $\psi(\vth{})$ are independent of $b_t$, $t\not\in O$, which means that $\nabla_{b_t}\psi = 0$. This makes sense: the linear function cannot get meaningful input for days $t$ where no observation is made. On the other hand, this implies that if a feature is active only at $t\not\in O$, its weight is not learned. Again, this makes sense. One cannot at the same time ignore observations and try to learn from them.

The criterion computation is easy to adjust. Namely, $\log P(\tvz{O})$ is obtained from the Gaussian filtering code if unobserved positions are taken into account. The gradient computations are adapted by simply replacing all sums over $t$ by sums over $t\in O$. In some cases, it is still simpler to compute a vector for all $t$, but then only use its $t\in O$ entries. Also note that $\partial_{b_t}\psi$ only gets contributions if $t\in O$ (see above).

For channel (b), we use $\sum_{t\in O}\partial_{\hsy{t}}\psi\cdot \partial_{\vth{}} \hsy{t}$. Here, we compute $\partial_{\vth{}}\hsy{t}$ for all $t$, but only use its $O$ entries.

For channel (c), define $\mxm{O} := \mxm{O,\cdot} =
\Id_{O,\cdot}\mxm{}$, which means a forward pass, followed by a
selection of $O$ components. In the Gaussian model for $\vxi{}$,
replace $\mxm{}$ by $\mxm{O}$. In particular, we only use
$\partial_{\hsy{t}}\psi$ for $t\in O$. Also, we use $\mxm{O}\vxi{}$
and $\partial_{\theta}\mxm{O}\vxi{}$. In either case, we first compute
these expressions for $\mxm{}$, by employing {\tt forwardPass}, then
select the $O$ components.

\subsubsection*{Gaussian Likelihood}

We may want to run the maximum likelihood learning code for a model with Gaussian likelihood, say $N(\tsz{t} | y_t,\sigma_t^2)$. In the most general case, both $\tsz{t}$ and $\sigma_t^2$ may depend on the parameter $\vth{}$. The criterion is
\[
\psi(\vth{}) = -\log\int \prod_t N(\tsz{t} | y_t,\sigma_t^2) P(\vy{}|\vs{})
P(\vs{})\, d\vy{} d\vs{}.
\]
No approximation is needed, the criterion can be computed by a single Kalman filtering call. For the gradient, we use that
\[
\partial_{p_t}\psi = \Ex\left[ \partial_{p_t}\tsphi{t}(y_t) \right],\quad
p_t\in\{ \tsz{t}, \sigma_t^2 \}.
\]
Fix $t$ and drop it from the notation. Some algebra gives
\[
\sigma^2\partial_{\tsz{}}\psi = \tsz{} - \Ex[y],\quad
\sigma^2\partial_{\sigma^2}\psi = \frac{1}2\left( 1 - \sigma^{-2}\Ex\left[
(\tsz{} - y)^2 \right] \right) = \frac{1}2\left( 1 - \sigma^{-2}\left(
(\tsz{} - \Ex[y])^2 + \Var[y] \right) \right).
\]
We compute two vectors $\ve[1]{}$, $\ve[2]{}$. Then, if $\theta$ is a
parameter of the Gaussian likelihood component:
\[
\partial_{\theta}\psi = \ve[1]{}{}^T\left[
\sigma_t^{-2}\partial_{\theta}\tsz{t} \right] + \ve[2]{}{}^T\left[
\sigma_t^{-2}\partial_{\theta}\sigma_t^2 \right].
\]

\bibliographystyle{plainnat}
\bibliography{latstate.bib}

\end{document}